\documentclass{article}

\usepackage[preprint]{neurips_2025}

\usepackage[utf8]{inputenc} 
\usepackage[T1]{fontenc}    
\usepackage{url}            
\usepackage{nicefrac}       
\usepackage{microtype}      

\usepackage{amsmath,amssymb,amsfonts,amsthm}
\usepackage{algorithmic}
\usepackage{graphicx}
\usepackage{textcomp}
\usepackage{color}
\usepackage[svgnames]{xcolor}
\PassOptionsToPackage{numbers,sort&compress}{natbib}
\definecolor{c1}{HTML}{2F70AF} 
\definecolor{pink}{HTML}{747199}
\newcommand{\reduce}[1]{\textcolor{pink}{#1}}
\usepackage[colorlinks=true,citecolor=pink,urlcolor=pink,linkcolor=pink]{hyperref}
\usepackage{enumitem}
\usepackage{algorithm}
\usepackage{threeparttable}
\usepackage{multirow,multicol}
\usepackage{booktabs}
\usepackage{setspace}
\usepackage{wrapfig}
\usepackage{subfigure}

\newcommand{\ie}{\textit{i}.\textit{e}., }

\newcommand{\eg}{\textit{e}.\textit{g}., }

 \newcommand{\linebreakand}{%
      \end{@IEEEauthorhalign}
      \hfill\mbox{}\par
      \mbox{}\hfill\begin{@IEEEauthorhalign}
    }

\newcommand{\D}{\mathrm{D}}
\newcommand{\K}{\mathrm{K}}
\newcommand{\mL}{\mathrm{L}}
\newcommand{\N}{\mathrm{N}}
\newcommand{\mP}{\mathrm{P}}
\newcommand{\mS}{\mathrm{S}}
\newcommand{\T}{\mathrm{T}}

\newcommand{\bst}[1]{{\textbf{\textcolor{red}{#1}}}}
\newcommand{\subbst}[1]{\textcolor{blue}{\underline{{#1}}}}
\newcommand{\scalea}[1]{\scalebox{1.0}{#1}}
\newcommand{\scaleb}[1]{\scalebox{1.0}{#1}}

\usepackage{caption}
\theoremstyle{plain}
\newtheorem{theorem}{Theorem}[section]

\theoremstyle{definition}

\theoremstyle{remark}

\title{Mixture of Low Rank Adaptation with Partial Parameter Sharing for Time Series Forecasting}

\author{
    \textbf{Licheng Pan$^{1}$\quad Zhichao Chen$^{1}$ \quad Haoxuan Li$^{2}$ \quad Guangyi Liu$^{1}$ \quad Zhijian Xu$^{3}$} \\ 
    \textbf{Zhaoran Liu$^{1}$ \quad Hao Wang$^{1}$ \quad Ying Wei$^{1}$} \\
    $^1$Zhejiang University \quad
    $^2$Peking University \quad 
    $^3$The Chinese University of Hong Kong \\
    \texttt{lc.pan@zju.edu.cn \quad Ho-ward@outlook.com}
}

\begin{document}

\maketitle

\begin{abstract}

Multi-task forecasting has become the standard approach for time-series forecasting (TSF). However, we show that it suffers from an \textsc{Expressiveness Bottleneck}, where predictions at different time steps share the same representation, leading to unavoidable errors even with optimal representations. To address this issue, we propose a two-stage framework: first, pre-train a foundation model for one-step-ahead prediction; then, adapt it using step-specific LoRA modules.
This design enables the foundation model to handle any number of forecast steps while avoiding the expressiveness bottleneck. We further introduce the \texttt{Mixture-of-LoRA (MoLA)} model, which employs adaptively weighted LoRA experts to achieve partial parameter sharing across steps. This approach enhances both efficiency and forecasting performance by exploiting interdependencies between forecast steps. Experiments show that MoLA significantly improves model expressiveness and outperforms state-of-the-art time-series forecasting methods. Code is available at~\url{https://anonymous.4open.science/r/MoLA-BC92}.

\end{abstract}

\section{Introduction}

Time-series forecasting (TSF) involves predicting future sequences based on observed ones and is widely applied in domains such as finance (\eg stock price prediction~\citep{fintsb}), meteorology (\eg weather forecasting~\citep{weather_app,weather_app2}), and manufacturing (\eg process monitoring~\citep{energy_app}). 
To construct effective TSF models, two questions warrant investigation: \textit{(1) How to extract informative representations from historical sequences; (2) How to generate accurate predictions using the extracted representations.}

Recent research has predominantly focused on learning representations from historical sequences, where the key is developing neural network architectures that can effectively capture the temporal dynamics within time-series data. Exemplary works in this area include recurrent neural networks~\citep{deepar}, convolutional neural networks~\citep{Timesnet, micn}, and graph neural networks~\citep{FourierGNN}.
Current progress is characterized by an ongoing debate between Transformers and simple linear models. Transformers, equipped with self-attention mechanisms, provide superior scalability~\citep{itransformer, PatchTST,fredformer}. In contrast, linear models, which encapsulate temporal dynamics using linear layers, are straightforward to implement and often demonstrate strong performance~\citep{DLinear,FBM,sparsetsf}.
These advancements showcase the rapid evolution of representation learning in the field of time-series data.

Despite significant progress in time-series representation learning, methods for generating predictions from these representations remain relatively underdeveloped. Specifically, current approaches predominantly adopt the multi-task forecasting (MT-F) paradigm, which employs a learnable linear layer to transform acquired representations into multi-step forecasts.
In this work, we reveal that this approach is inherently constrained by an expressiveness bottleneck. Specifically, forecasts for different time steps are restricted to being linear combinations of shared representation bases. This limitation significantly constrains model capacity, leading to inevitable forecasting errors regardless of the quality of the learned representations.

To mitigate the expressiveness bottleneck, we propose a two-stage framework for constructing time-series forecasting models.
It involves first training a single-step forecast model, and then adapting it to other forecast horizons by training step-specific LoRA (Low-Rank Adaptation) modules. However, directly applying this approach neglects the interdependencies between different forecast steps, resulting in huge computational costs and suboptimal performance.
To mitigate these issues, we introduce the Mixture-of-LoRA (MoLA) model. This approach shares partial parameters among step-specific LoRA modules, which effectively reduces computational cost and enhances forecasting performance.

Our main contributions are summarized as follows:
\begin{itemize}[leftmargin=*]
    \item We formalize the expressiveness bottleneck in the prevalent MT-F paradigm, which limits the capacity of forecasting models. We further propose a two-stage framework to construct time-series forecasting models free from this bottleneck.
    \item We introduce MoLA to implement the two-stage framework. Building on the standard LoRA technique, MoLA enables partial parameter sharing among LoRA modules across different forecast steps, leveraging interdependencies between these steps to enhance both efficiency and accuracy.
    \item We validate the efficacy of MoLA through extensive experiments, demonstrating its substantial superiority over state-of-the-art methods trained from scratch under the same forecasting lengths.
\end{itemize}

\section{Preliminaries}
\label{subsec:TSF}

In this work, we consider the time-series forecasting (TSF) problem: predicting future observations from historical data. Given a time-series dataset $X$ with $\D$ covariates, where $X_n$ denotes the observation at the $n$-th step, we define two key elements~\citep{box2015time}: (1) \textbf{Historical sequence} $L = [X_{n-\mL+1}, \ldots, X_n] \in \mathbb{R}^{\mL \times \D}$, where $\mL$ is the historical window length; (2) \textbf{Label sequence} $Y = [X_{n+1}, \ldots, X_{n+\T}] \in \mathbb{R}^{\T \times \D}$, where $\T$ is the forecast horizon. Based on these elements, TSF can be described as estimating $\mathbb{E}[Y|L]$—the expected label sequence conditioned on the past~\citep{nguyen2024scaling,ghimire2024two}.

Most TSF models follow a two-stage architecture. The encoder $g_\mathrm{e}$ extracts informative representations from historical sequence: $R=g_\mathrm{e}(L)$; the decoder $g_\mathrm{d}$ then transforms $R$ to generate forecasts: $\hat{Y}=g_\mathrm{d}(R)$.
While both stages are essential, recent research has mainly advanced encoder designs, developing diverse architectures such as convolutional neural networks~\citep{Timesnet, micn}, graph neural networks~\citep{FourierGNN}, and Transformers~\citep{itransformer,Transformer}.  In contrast, the decoder is understudied despite its importance.

There are two mainstream decoding approaches for TSF. Early methods adopt the autoregressive forecasting (AR-F) approach, where a linear layer with a single output produces one-step-ahead predictions, and the forecast sequence is generated sequentially by feeding previous outputs back as inputs. However, this approach suffers from error accumulation: errors in earlier steps interfere with later predictions.
Modern TSF methods favor the multi-step forecasting (MT-F) paradigm~\citep{Informer,DLinear,itransformer}. As illustrated in Figure~\ref{subfig:MT-F}, a linear layer with $\T$ outputs is applied to $R$ to generate the entire forecast sequence at once. This direct approach reduces error accumulation and eliminates the sequential computation overhead, making it the preferred choice in current TSF models.

\section{Proposed Method}
\subsection{Motivation}

MT-F has been a pervasive approach for generating forecasts from learned representations~\citep{itransformer,Timesnet,DLinear}, as it mitigates error accumulation and eliminates the sequential computation overhead~\citep{Informer}. Here, we formulate the MT-F workflow and identify an expressiveness bottleneck limiting its capacity.

Let $R\in\mathbb{R}^{\mL\times\mathrm{D}}$ be the encoder output. MT-F generates the forecast sequence using a linear layer with learnable weights 
$W\in\mathbb{R}^{\T\times\mL}$ and $b\in\mathbb{R}^{\T}$:
\begin{equation}
    \left[\hat{Y}_1, \hat{Y}_2, \ldots, \hat{Y}_\T\right] = \left[W_1, W_2,\ldots,  W_\T\right] R + [b_1, b_2,...,b_\T].
\end{equation}
where $\hat{Y}_t=W_t R + b_t$ is the prediction for the $t$-th future step. This  immediately assumes that different tasks (\ie forecast steps) share the same optimal representation ($R$), such that applying different linear weights yields accurate forecasts. However, in practice, the optimal representation for each step may differ substantially, especially when the forecast horizon is large (\eg $\T=720$).
This mismatch results in an \textit{expressiveness bottleneck}: regardless of encoder quality, using a shared representation across all forecast steps leads to an unavoidable modeling error. As shown in Theorem~\ref{thm:bottleneck_quantification}, this error is strictly positive when $\mathrm{L}+1 < \mathrm{T}$ --- a typical setting in the long-term TSF task.





\begin{theorem}[Expressiveness Bottleneck]
\label{thm:bottleneck_quantification}
Let $\bar{W} = [W \ b] \in \mathbb{R}^{\T \times (\mL+1)}$ be the parameters in the MT-F's linear decoder, $Y\in\mathbb{R}^{\T\times\mathrm{D}}$ be the label sequence;  the minimum attainable estimation error is
\begin{equation}
    \|\epsilon\|_F^2 = \sum_{t=\mathrm{rank}(\bar{W})+1}^{\T} \|U_t^\top Y\|_2^2,
\end{equation}
where $\bar{W} = U\Sigma V^\top$ is the singular value decomposition of $\bar{W}$,
$\mathrm{rank}(\bar{W}) \leq \min\{\T, \mL+1\}$, and $\{U_i\}_{i=\mathrm{rank}+1}^{\T}$ form an orthonormal basis for the null space of $\bar{W}$. Notably, this error is independent of the representation $R$ provided by encoder.
\end{theorem}


\begin{figure}[t]
\begin{center}
\subfigure[Synthetic data]{\includegraphics[width=0.245\linewidth]{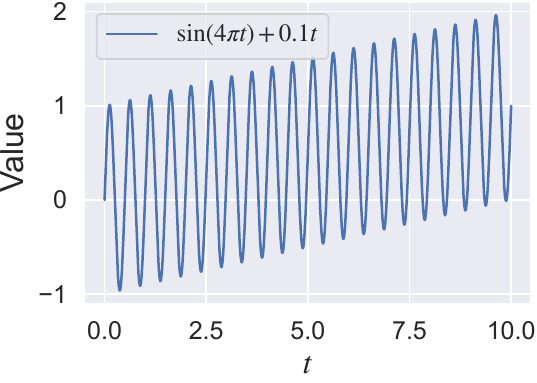}\label{subfig:syn_data}}
\hfill
\subfigure[T=1]{\includegraphics[width=0.245\linewidth]{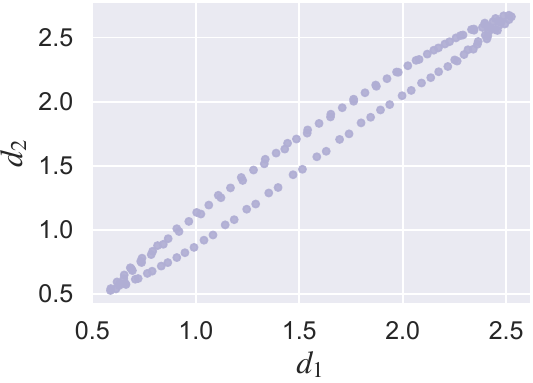}}
\hfill
\subfigure[T=16]{\includegraphics[width=0.245\linewidth]{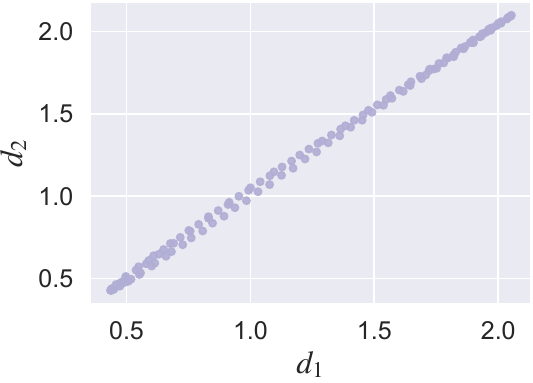}}
\hfill
\subfigure[T=32]{\includegraphics[width=0.245\linewidth]{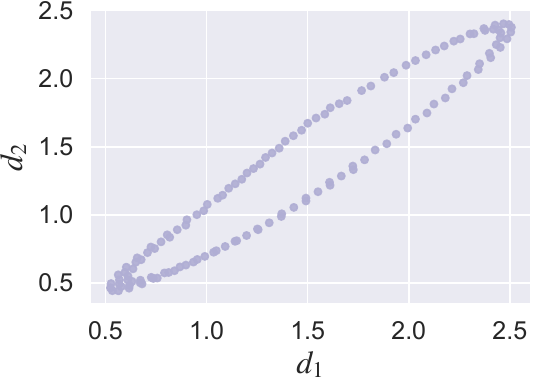}}
\caption{Visualization of representations generated with different forecasting step.}
\label{fig:toy}
\end{center}
\vspace{-5mm}
\end{figure}

\paragraph{Case Study.} To illustrate the expressiveness bottleneck, we conduct a case study on a synthetic dataset. We set the historical window length to $\mL=16$. The encoder is a two-layer perceptron with hidden sizes 16 and 2; the decoder is a linear layer with a single output. For each forecast step, we train a separate single-output model and visualize the encoder outputs in Figure~\ref{fig:toy} (b)-(d). The results show that the optimal representations differ significantly across steps, demonstrating that sharing a single representation across all forecast steps leads to unavoidable modeling errors.

Given the decoder’s critical role and the limitations of current approaches, there is a clear need for a decoding strategy that can overcome the expressiveness bottleneck. At the same time, the primary advantage of MT-F—parameter sharing across forecast steps—has proven beneficial for accuracy and should be preserved. To this end, two key questions warrant investigation: \textit{How to design a decoding strategy that avoids the expressiveness bottleneck while maintaining parameter sharing? Does addressing this bottleneck actually improve forecasting performance?}

\begin{figure}[t]
\begin{center}
\subfigure[MT-F]{
    \includegraphics[width=0.141\linewidth]{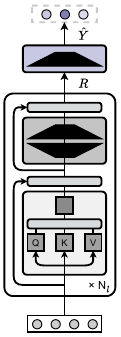}
    \label{subfig:MT-F}
}
\hfill
\raisebox{0.0\height}{\color{gray}\rule{0.8pt}{5.4cm}}
\hfill
\subfigure[MF]{
    \includegraphics[width=0.28\linewidth]{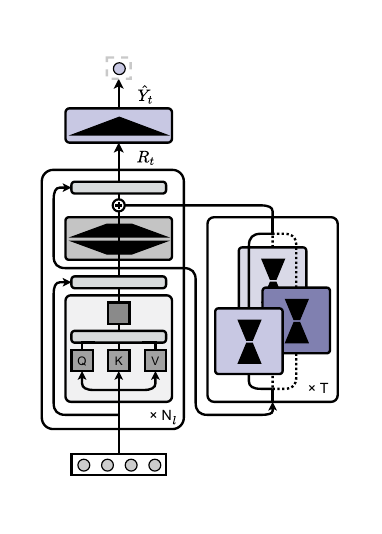}
    \label{subfig:mf}
}
\hfill
\raisebox{0.0\height}{\color{gray}\rule{0.8pt}{5.4cm}}
\hfill
\subfigure[MoLA]{
    \includegraphics[width=0.395\linewidth]{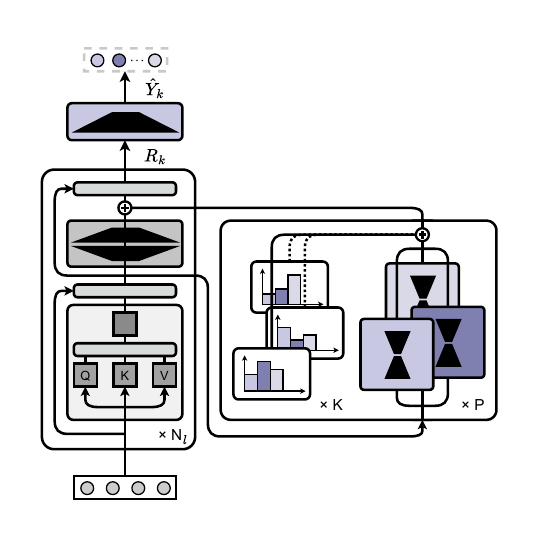}
    \label{subfig:MoLA}
}
\caption{Visualization of MT-F, LoRA and MoLA approaches to generate multi-step forecasts. Gray blocks denote identical encoder components. Purple blocks represent decoding strategies. Rectangles with varying transparencies indicate different expert matrices in MoLA.}
\label{fig:paradigms}
\end{center}
\vspace{-5mm}
\end{figure}

\subsection{A Two-Stage Framework for Breaking Expressiveness Bottleneck}

Motivated by Theorem~\ref{thm:bottleneck_quantification}, which shows that the error term disappears when the linear layer's output dimension is 1, we advocate a two-stage framework to address the expressiveness bottleneck. We first pre-train a foundation model with a single output for one-step-ahead prediction and subsequently adapt it to multiple forecasting steps. The adaptation follows Low-Rank Adaptation (LoRA), which excels at adapting foundation models to different tasks without hampering pre-training performance.

\paragraph{Pre-training.} We begin by training a foundation model for one-step prediction. Given the historical sequence $L \in \mathbb{R}^{\mL \times \mathrm{D}}$, the process can be described as follows:
\begin{equation}
    R = g_\mathrm{e}(L ),  \quad \hat{Y}_1 = g_\mathrm{d}(R) = W R + b,\quad  \mathcal{L}_\mathrm{pre} = \| Y_1 - \hat{Y}_1 \|_2^2,
\end{equation}
where the encoder $g_\mathrm{e}$ extracts features $R$, and the one-step-ahead prediction $\hat{Y}_{1} \in \mathbb{R}^\mathsf{D}$ is generated using a single-output linear projection with weights $W\in \mathbb{R}^{1\times\mL}$ and bias $b\in\mathbb{R}$. The pre-training objective is the MSE on the one-step prediction.

\paragraph{Adaptation.} After pre-training for one-step-ahead prediction ($\mathrm{T} = 1$), we adapt the model for multi-step forecasting  ($\mathrm{T} \geq 2$). To this end, we introduce LoRA modules into selected encoder linear layers, while freezing the weights of the original encoder and decoder.  
Specifically, for a selected linear layer with weight matrix $M \in \mathbb{R}^{\mathrm{d}_\mathrm{out} \times \mathrm{d}_\mathrm{in}}$, we modify the weights for each step $t$ as follows:
\begin{equation}  
    M^{(t)} = M + B^{(t)} A^{(t)},  
\end{equation}  
where $B^{(t)}\in \mathbb{R}^{\mathrm{d}_\mathrm{out} \times r}$ and $A^{(t)}\in \mathbb{R}^{r\times\mathrm{d}_\mathrm{in}}$ are LoRA matrices, with rank $r < \min(\mathrm{d}_\mathrm{out}, \mathrm{d}_\mathrm{in})$.  For each forecast step $t$, the modified encoder $g_\mathrm{e}^{(t)}$ generates predictions as follows:
\begin{equation}
    R^{(t)} = g_\mathrm{e}^{(t)}(L),  \quad \hat{Y}_t = g_\mathrm{d}(R^{(t)}) = W R + b,\quad  \mathcal{L}_\mathrm{ada}^{(t)} = \| Y_t - \hat{Y}_t \|_2^2,
\end{equation}
where the adaptation objective is the MSE on the $t$-step prediction. Notably, only the LoRA matrices ($B^{(t)}, A^{(t)}$) are optimized to minimize the adaptation objective, while the rest of the foundation model remains frozen. This process is repeated for each prediction step $1<t\leq\T$, adapting the foundation model to support diverse forecast steps.

\paragraph{Discussion. } The framework overcomes the expressiveness bottleneck by ensuring both the pretrained and adapted TSF models use a single output, thereby eliminating the error term as shown in Theorem~\ref{thm:bottleneck_quantification}. Moreover, during the adaptation stage, the pre-trained weights remain frozen, preventing any interference with the foundation model. This property has a critical implication for TSF: adapting to long-range forecasts does not compromise the performance on short-term predictions.

\subsection{Mixture of Low-rank Adaptation}
Time series data exhibits strong autocorrelation, inducing correlations across multiple forecasting steps. However, the above approach treats each forecast step independently by assigning separate LoRA matrices to generate representation and using a single-output layer for prediction. This leads to two limitations: (1) the number of adaptations increases linearly with the forecast horizon, resulting in high computational cost; and (2) neglecting inter-step correlations can impair forecasting accuracy.

In this section, we propose MoLA, a modified LoRA technique that leverages label correlations for improved performance and efficiency. According to Theorem~\ref{thm:bottleneck_quantification}, the expressiveness bottleneck vanishes as long as the number of outputs does not exceed $\mathrm{L+1}$; thus, \textit{single-step output is not strictly required}. Building on this insight, we partition the $\T$-step label sequence into $\K$ contiguous segments, each containing $\mS = \T/\K$ steps. The foundation model’s output dimension is then increased from 1 to $\mS$, allowing each adaptation to adapt $\mS$ future steps jointly, and thereby reducing the number of required adaptations. This segment-based adaptation exploits intra-segment label dependencies to improve accuracy and efficiency, while bypassing expressiveness bottleneck as long as $\mS\leq \mathrm{L+1}$.

To further exploit label correlation across segments, we enable partial parameter sharing across different LoRA modules via a mixture-of-experts mechanism. Specifically, suppose we have $\mathrm{P}$ expert matrices, each denoted as $B^{(p)}$ and $A^{(p)}$.  For a selected linear layer in the encoder with original weights $M$, the adaptation for segment $k$ is given by
\begin{equation}
\label{eq:MoLA-forward}
    M^{(k)} = M + \sum_{p=1}^{\mP} \Delta^{(p)}_k \cdot B^{(p)} A^{(p)},
\end{equation}
where $\Delta_k=[\Delta^{(1)}_k,\Delta^{(2)}_k,\ldots,\Delta^\mathrm{(P)}_k]$ is a learnable, normalized weight vector that adaptively combines the $\mP$ expert matrices for each segment. 
This design enables expert matrices to be shared across segments, allowing the model to leverage inter-segment label correlations and enrich supervision signals for improved adaptation. A visual comparison among MT-F, LoRA, and MoLA is provided in Figure~\ref{fig:paradigms}.



\subsection{Overall Workflow}


\begin{wrapfigure}{r}{0.5\linewidth}
\centering
\vspace{-8mm}
\begin{minipage}{\linewidth}
\begin{algorithm}[H]
\setstretch{1}
\algsetup{linenosize=\scriptsize}
\footnotesize
\caption{Mixture of Low-rank Adaptation for TSF.}
\label{algo:MoLA}
\flushleft
\textbf{Input}: $X$: training dataset, $\K$: number of segments, $\mP$: number of LoRA experts, $r$: rank of LoRA matrices. \\
\textbf{Parameter}: 
$\{\Delta_k\}_{k=1}^{\K}$: segment-specific weight vectors, 
$\{B^{(p)}, A^{(p)}\}_{p=1}^{\mP}$: LoRA expert matrices. \\
\textbf{Output}: 
$\{\Delta_k\}_{k=1}^{\K}$, $\{B^{(p)}, A^{(p)}\}_{p=1}^{\mP}$: optimized adaptation parameters. \\
\begin{algorithmic}[1] 
\STATE Pre-train a $\mS$-step model $\{g_\mathrm{e},g_\mathrm{d}\}$ on $X$
\STATE Freeze the foundation model parameters
\FOR{$k = 1$ to $\K$}
    \FOR{$L,Y$ in $X$}
        \STATE Modify the encoder $g_\mathrm{e}$ to $g_\mathrm{e}^{(k)}$: \\
        $M^{(k)} \leftarrow M + \sum_{p=1}^{\mP} \Delta^{(p)}_k \cdot B^{(p)}A^{(p)}$
        \STATE $R^{(k)} \leftarrow g_\mathrm{e}^{(k)}(L), \hat{Y}_k \leftarrow g_\mathrm{d}(R^{(k)})$
        \STATE $\mathcal{L}_\mathrm{ada}^{(k)} \leftarrow ||Y_k - \hat{Y}_k||_2^2$
        \STATE Update $\Delta_k$ and $\{B^{(p)}, A^{(p)}\}_{p=1}^\mP$ with $\mathcal{L}_\mathrm{ada}^{(k)}$
    \ENDFOR
    \STATE Freeze the segment-specific weight vector $\Delta_k$
\ENDFOR
\STATE Return optimized adaptation parameters
\end{algorithmic}
\end{algorithm}
\label{fig:framework}
\vspace{-5mm}
\end{minipage}
\end{wrapfigure}

In this section, we detail the procedure for applying MoLA to train TSF models, as outlined in Algorithm~\ref{algo:MoLA}. The workflow begins by pre-training a foundation model for $\mS$-step prediction (step 1). After pre-training, the model parameters are frozen (step 2).
Next, we inject adaptation modules into the backbone: for each of the $\K$ segments, a segment-specific weight vector and $\mP$ shared LoRA expert modules are introduced. For each segment, the mixing weights are learned to adaptively combine the shared LoRA experts, as defined in~\eqref{eq:MoLA-forward} (step 5). During fine-tuning, the segment-specific weight vector and segment-shared LoRA modules are optimized using the corresponding segment's ground truth, while the backbone remains frozen (steps 6-8).
And during inference, the segment-specific adapters are applied to the frozen foundation model to generate forecasts for each segment, which are then concatenated to produce the final estimations.

\section{Experiments}
To validate the effectiveness of our MoLA approach for TSF, we conduct a comprehensive empirical evaluation across six key dimensions:

\begin{enumerate}[leftmargin=*]
    \item \textbf{Performance:} \textit{How does MoLA perform compared to the current state-of-the-art?} 
    In Section~\ref{subsec:overall}, we benchmark MoLA against state-of-the-art baselines using public datasets to assess its overall performance.
    
    \item \textbf{Superiority:} \textit{Is MoLA more effective than other forecasting paradigms?}
    Section~\ref{subsec:paradigm} provides a comparative analysis of MoLA versus AR-F and MT-F paradigms, demonstrating its advantages in forecasting accuracy.
    
    
    \item \textbf{Generality:} \textit{Can MoLA be applied to other forecasting models?}
    In Section~\ref{subsec:gene_model}, we assess the versatility of MoLA by applying it to various forecasting models, illustrating its broad applicability.
    
    \item \textbf{Flexibility:} \textit{Does MoLA accommodate different fine-tuning modules besides LoRA?}
    Section~\ref{subsec:gene_module} explores MoLA's flexibility by substituting LoRA with alternative fine-tuning modules, demonstrating its adaptability in implementation.
    
    \item \textbf{Sensitivity:} \textit{How sensitive is MoLA to hyperparameter changes?}
    In Section~\ref{subsec:sensi}, we conduct a sensitivity analysis of the hyperparameter $\K$, showing that MoLA maintains robust performance across a wide range of parameter settings.
    
\end{enumerate}

\subsection{Experimental Setup}
\paragraph{Datasets.}
We utilize a diverse set of benchmark datasets for time-series forecasting, including ETT (with 4 subsets), ECL, Traffic, Weather, and PEMS (with 2 subsets), following the experimental settings used in~\citep{Autoformer} and~\citep{itransformer}. Each dataset is split chronologically into training, validation, and test sets to preserve temporal order. Detailed descriptions of each dataset, including their characteristics and split ratios, can be found in Appendix D.1.

\paragraph{Baselines.}
Our baselines consist of a wide range of established models from the time-series forecasting domain: Transformer~\citep{Transformer}, Informer~\citep{Informer}, Autoformer~\citep{Autoformer}, FEDformer~\citep{fedformer}, iTransformer~\citep{itransformer}, DLinear~\citep{DLinear} and FreTS~\citep{FreTS}, TimesNet~\citep{Timesnet} and TCN~\citep{tcn}.

\paragraph{Implementation details.}
To ensure consistency and fairness, we faithfully reproduced all baseline models using training scripts from the official TimesNet repository\footnote{\url{https://github.com/thuml/Time-Series-Library}.}. Optimization was performed using the Adam optimizer~\cite{Adam}, with learning rates selected from $\{0.0001, 0.0005, 0.001\}$. The forecasting horizons were set to 96, 192, 336, and 720 for the ETT, ECL, Traffic, and Weather datasets, and 12, 24, 36, and 48 for the PEMS datasets.
We adopted mean square error (MSE) and mean absolute error (MAE) as primary evaluation metrics, in line with standard practices in the field. 

To integrate the MoLA paradigm into existing models, we adhered to the original hyperparameter configurations of each baseline during the pre-training phase.
Forecasting horizons of the pre-trained models were set to $\T / \K$, where $\K \in \{2, 3, 4, 6\}$. In the fine-tuning phase, we only tuned the rank $r$ in the LoRA expert module, the learning rate $\eta$, and the number of LoRA experts, ensuring minimal deviation from the original model structures. 
All experiments were conducted on Intel(R) Xeon(R) Platinum 8383C CPUs and NVIDIA RTX 4090 GPUs. Additional implementation details are provided in Appendix D.2.

\begin{table}
  \caption{Multi-step forecasting performance.}\label{tab:multistep_app}
  \centering
  \begin{threeparttable}
  \renewcommand{\arraystretch}{1.0}
  \setlength{\tabcolsep}{1.6pt}
  \tiny
  \renewcommand{\multirowsetup}{\centering}
  \begin{tabular}{c|c|cc|cc|cc|cc|cc|cc|cc|cc|cc|cc|cc}
    \toprule
    \multicolumn{2}{l}{\multirow{2}{*}{\rotatebox{0}{\scaleb{Models}}}} & 
    \multicolumn{2}{c}{\rotatebox{0}{\scaleb{\textbf{MoLA}}}} &
    \multicolumn{2}{c}{\rotatebox{0}{\scaleb{iTransformer}}} &
    \multicolumn{2}{c}{\rotatebox{0}{\scaleb{FreTS}}} &
    \multicolumn{2}{c}{\rotatebox{0}{\scaleb{TimesNet}}} &
    \multicolumn{2}{c}{\rotatebox{0}{\scaleb{TiDE}}} &
    \multicolumn{2}{c}{\rotatebox{0}{\scaleb{DLinear}}} &
    \multicolumn{2}{c}{\rotatebox{0}{\scaleb{FEDformer}}} &
    \multicolumn{2}{c}{\rotatebox{0}{\scaleb{Autoformer}}} &
    \multicolumn{2}{c}{\rotatebox{0}{\scaleb{Informer}}} &
    \multicolumn{2}{c}{\rotatebox{0}{\scaleb{Transformer}}} &
    \multicolumn{2}{c}{\rotatebox{0}{\scaleb{TCN}}} \\
    \multicolumn{2}{c}{} &
    \multicolumn{2}{c}{\scaleb{\textbf{(Ours)}}} & 
    \multicolumn{2}{c}{\scaleb{(2024)}} & 
    \multicolumn{2}{c}{\scaleb{(2023)}} & 
    \multicolumn{2}{c}{\scaleb{(2023)}} & 
    \multicolumn{2}{c}{\scaleb{(2023)}} & 
    \multicolumn{2}{c}{\scaleb{(2023)}} & 
    \multicolumn{2}{c}{\scaleb{(2022)}} &
    \multicolumn{2}{c}{\scaleb{(2021)}} &
    \multicolumn{2}{c}{\scaleb{(2021)}} &
    \multicolumn{2}{c}{\scaleb{(2017)}} &
    \multicolumn{2}{c}{\scaleb{(2017)}} \\
    \cmidrule(lr){3-4} \cmidrule(lr){5-6}\cmidrule(lr){7-8} \cmidrule(lr){9-10}\cmidrule(lr){11-12} \cmidrule(lr){13-14} \cmidrule(lr){15-16} \cmidrule(lr){17-18} \cmidrule(lr){19-20} \cmidrule(lr){21-22} \cmidrule(lr){23-24}
    \multicolumn{2}{l}{\rotatebox{0}{\scaleb{Metrics}}}  & \scalea{MSE} & \scalea{MAE}  & \scalea{MSE} & \scalea{MAE}  & \scalea{MSE} & \scalea{MAE}  & \scalea{MSE} & \scalea{MAE}  & \scalea{MSE} & \scalea{MAE}  & \scalea{MSE} & \scalea{MAE} & \scalea{MSE} & \scalea{MAE} & \scalea{MSE} & \scalea{MAE} & \scalea{MSE} & \scalea{MAE} & \scalea{MSE} & \scalea{MAE} & \scalea{MSE} & \scalea{MAE} \\
    \midrule

    \multicolumn{2}{l}{\scalea{ETTm1}}
    & \scalea{\bst{0.400}} & \scalea{\bst{0.406}} 
    & \scalea{0.415} & \scalea{0.416} 
    & \scalea{0.407} & \scalea{0.415} 
    & \scalea{0.413} & \scalea{0.418} 
    & \scalea{0.419} & \scalea{0.419} 
    & \scalea{\subbst{0.404}} & \scalea{\subbst{0.407}} 
    & \scalea{0.440} & \scalea{0.451} 
    & \scalea{0.596} & \scalea{0.517}
    & \scalea{0.887} & \scalea{0.693}
    & \scalea{0.943} & \scalea{0.733} 
    & \scalea{0.891} & \scalea{0.632} \\
    \midrule

    \multicolumn{2}{l}{\scalea{ETTm2}}
    & \scalea{\bst{0.287}} & \scalea{\bst{0.330}} 
    & \scalea{\subbst{0.294}} & \scalea{0.335} 
    & \scalea{0.335} & \scalea{0.379} 
    & \scalea{0.297} & \scalea{\subbst{0.332}} 
    & \scalea{0.358} & \scalea{0.404}
    & \scalea{0.344} & \scalea{0.396} 
    & \scalea{0.302} & \scalea{0.348} 
    & \scalea{0.326} & \scalea{0.366} 
    & \scalea{1.256} & \scalea{0.801}
    & \scalea{1.322} & \scalea{0.814} 
    & \scalea{3.411} & \scalea{1.432} \\
    \midrule

    \multicolumn{2}{l}{\scalea{ETTh1}}
    & \scalea{\subbst{0.443}} & \scalea{\bst{0.441}} 
    & \scalea{0.449} & \scalea{\subbst{0.447}} 
    & \scalea{0.488} & \scalea{0.474} 
    & \scalea{0.478} & \scalea{0.466} 
    & \scalea{0.628} & \scalea{0.574} 
    & \scalea{0.462} & \scalea{0.458} 
    & \scalea{\bst{0.441}} & \scalea{0.457} 
    & \scalea{0.476} & \scalea{0.477} 
    & \scalea{1.064} & \scalea{0.806}
    & \scalea{0.993} & \scalea{0.788} 
    & \scalea{0.763} & \scalea{0.636} \\
    \midrule

    \multicolumn{2}{l}{\scalea{ETTh2}}
    & \scalea{\bst{0.377}} & \scalea{\bst{0.401}} 
    & \scalea{\subbst{0.390}} & \scalea{\subbst{0.410}} 
    & \scalea{0.550} & \scalea{0.515} 
    & \scalea{0.413} & \scalea{0.426} 
    & \scalea{0.611} & \scalea{0.550} 
    & \scalea{0.558} & \scalea{0.516} 
    & \scalea{0.430} & \scalea{0.447} 
    & \scalea{0.478} & \scalea{0.483} 
    & \scalea{4.358} & \scalea{1.719}
    & \scalea{3.296} & \scalea{1.419} 
    & \scalea{3.325} & \scalea{1.445} \\
    \midrule

    \multicolumn{2}{l}{\scalea{ECL}}
    & \scalea{\bst{0.174}} & \scalea{\bst{0.265}} 
    & \scalea{\subbst{0.176}} & \scalea{\subbst{0.267}} 
    & \scalea{0.209} & \scalea{0.297} 
    & \scalea{0.214} & \scalea{0.307} 
    & \scalea{0.251} & \scalea{0.344} 
    & \scalea{0.225} & \scalea{0.319} 
    & \scalea{0.229} & \scalea{0.339} 
    & \scalea{0.228} & \scalea{0.339} 
    & \scalea{0.335} & \scalea{0.416}
    & \scalea{0.274} & \scalea{0.367} 
    & \scalea{0.617} & \scalea{0.598} \\
    \midrule

    \multicolumn{2}{l}{\scalea{Traffic}}
    & \scalea{\bst{0.425}} & \scalea{\bst{0.284}} 
    & \scalea{\subbst{0.428}} & \scalea{\subbst{0.286}} 
    & \scalea{0.552} & \scalea{0.348} 
    & \scalea{0.535} & \scalea{0.309} 
    & \scalea{0.760} & \scalea{0.473} 
    & \scalea{0.673} & \scalea{0.419} 
    & \scalea{0.611} & \scalea{0.379} 
    & \scalea{0.637} & \scalea{0.399} 
    & \scalea{0.727} & \scalea{0.404}
    & \scalea{0.680} & \scalea{0.376} 
    & \scalea{1.001} & \scalea{0.652} \\
    \midrule

    \multicolumn{2}{l}{\scalea{Weather}}
    & \scalea{0.266} & \scalea{\bst{0.286}} 
    & \scalea{0.281} & \scalea{0.302} 
    & \scalea{\bst{0.255}} & \scalea{0.299} 
    & \scalea{\subbst{0.262}} & \scalea{\subbst{0.288}} 
    & \scalea{0.271} & \scalea{0.320} 
    & \scalea{0.265} & \scalea{0.317} 
    & \scalea{0.311} & \scalea{0.361} 
    & \scalea{0.349} & \scalea{0.391} 
    & \scalea{0.595} & \scalea{0.541}
    & \scalea{0.632} & \scalea{0.552} 
    & \scalea{0.584} & \scalea{0.572} \\
    \midrule

    \multicolumn{2}{l}{\scalea{PEMS03}}
    & \scalea{\bst{0.112}} & \scalea{\bst{0.222}} 
    & \scalea{\subbst{0.116}} & \scalea{0.226} 
    & \scalea{0.146} & \scalea{0.257} 
    & \scalea{0.118} & \scalea{\subbst{0.223}} 
    & \scalea{0.316} & \scalea{0.370} 
    & \scalea{0.233} & \scalea{0.344} 
    & \scalea{0.174} & \scalea{0.302} 
    & \scalea{0.501} & \scalea{0.513} 
    & \scalea{0.137} & \scalea{0.241}
    & \scalea{0.126} & \scalea{0.233} 
    & \scalea{0.666} & \scalea{0.634} \\
    \midrule

    \multicolumn{2}{l}{\scalea{PEMS08}}
    & \scalea{\bst{0.138}} & \scalea{\bst{0.236}} 
    & \scalea{0.159} & \scalea{0.258} 
    & \scalea{0.174} & \scalea{0.277} 
    & \scalea{\subbst{0.154}} & \scalea{\subbst{0.245}} 
    & \scalea{0.319} & \scalea{0.378} 
    & \scalea{0.294} & \scalea{0.377} 
    & \scalea{0.232} & \scalea{0.322} 
    & \scalea{0.630} & \scalea{0.572} 
    & \scalea{0.319} & \scalea{0.314}
    & \scalea{0.249} & \scalea{0.266} 
    & \scalea{0.713} & \scalea{0.629} \\
    \bottomrule
  \end{tabular}
  \begin{tablenotes}
    \item  \scriptsize \textit{Note}:  We fix the input length as 96 following the established benchmarks~\citep{itransformer,Timesnet}. \bst{Bold} typeface highlights the top performance for each metric, while \subbst{underlined} text denotes the second-best results. The results are averaged over prediction lengths (96, 192, 336 and 720), with full results in Table~\ref{tab:multistep_app_full}.
  \end{tablenotes}
  \end{threeparttable}
  \vspace{-3mm}
\end{table}

\subsection{Overall Performance}
\label{subsec:overall}
The performance of our proposed modularized fine-tuning approach on the MSTF task is presented in Table~\ref{tab:multistep_app}. We use iTransformer as the base model $g_{\theta}$ and apply it to different forecasting horizons using the two-stage MoLA paradigm.
Overall, MoLA significantly enhances the performance of iTransformer. For example, on the ETTm1 dataset, MoLA reduces the MSE of iTransformer by 0.015. Similar improvements are observed across other datasets, underscoring the effectiveness of the MoLA paradigm in modularized and personalized fine-tuning for different forecasting horizons.

Importantly, MoLA not only improves iTransformer’s performance but also enables it to surpass models that previously outperformed iTransformer on certain datasets and metrics, like PEMS08 with MSE and MAE.
This suggests that the gains achieved by MoLA go beyond architectural designs alone, emphasizing the importance of modularized and personalized fine-tuning in addressing the expressiveness bottleneck often encountered in the MT-F paradigm.

\begin{figure}
\begin{center}
\subfigure[\hspace{-20pt}]{\includegraphics[width=0.245\linewidth]{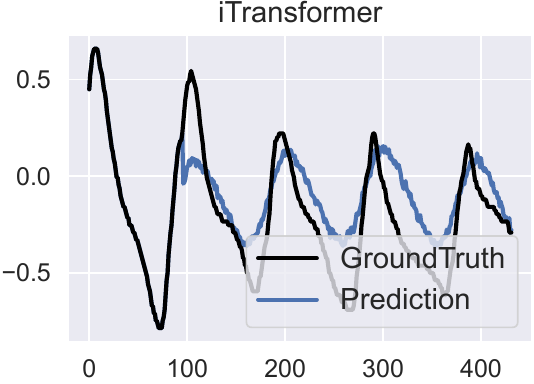}}
\hfill
\subfigure[\hspace{-20pt}]{\includegraphics[width=0.245\linewidth]{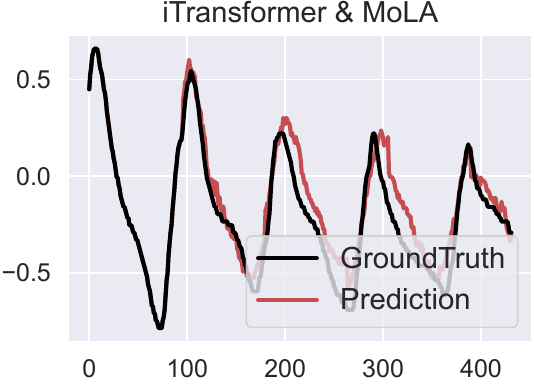}}
\hfill
\subfigure[\hspace{-20pt}]{\includegraphics[width=0.245\linewidth]{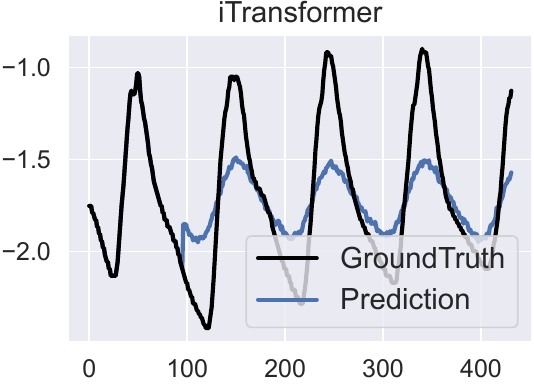}}
\hfill
\subfigure[\hspace{-20pt}]{\includegraphics[width=0.245\linewidth]{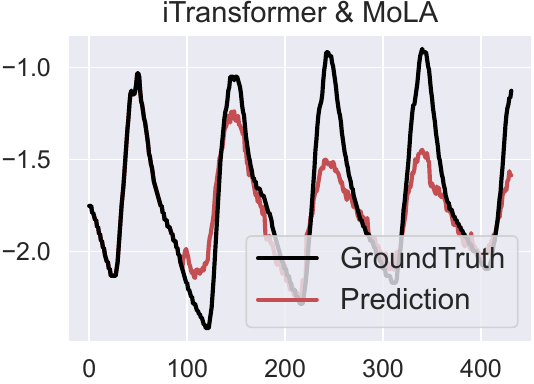}}
\caption{Visualization of forecast sequence generated with and without MoLA under two snapshots.}
\label{fig:case}
\end{center}
\vspace{-5mm}
\end{figure}


\paragraph{Showcases.}
To further illustrate the improvements provided by MoLA, we visualize forecast sequences for two snapshots from the ETTm2 dataset with a forecasting horizon of $\T=336$ in Figure~\ref{fig:case}. While the MT-F paradigm follows the general trends of the true label sequence, it struggles with capturing sharp peaks, leading to misaligned forecasts across different horizons. In contrast, MoLA mitigates this issue by employing horizon-specific fine-tuning, resulting in more accurate predictions that track both trends and sharp peaks across multiple horizons.

\subsection{Comparative Studies}
\label{subsec:paradigm}
\begin{wraptable}{r}{0.5\linewidth}
\vspace{-4mm}
\caption{Comparative study results.}
\label{tab:compara}
\renewcommand{\arraystretch}{1.3} 
\centering
\begin{threeparttable}
\setlength{\tabcolsep}{1.4pt}
\tiny
\begin{tabular}{l|l|cr|cr|cr|cr}
    \hline
    \multicolumn{2}{l|}{} & \multicolumn{4}{c|}{ETTh1} & \multicolumn{4}{c}{Weather} \\
    \hline
    \multicolumn{2}{l|}{Paradigm} & MSE & \multicolumn{1}{c|}{$\Delta$} & MAE & \multicolumn{1}{c|}{$\Delta$} & MSE & \multicolumn{1}{c|}{$\Delta$} & MAE & \multicolumn{1}{c}{$\Delta$} \\
    \hline
    
    \multirow{3}{*}{{\rotatebox{90}{\scalebox{0.95}{96}}}}
    & \scalea{AR-F}
    & 0.454 & \multicolumn{1}{c|}{-} & 0.441 & \multicolumn{1}{c|}{-}
    & 0.235 & \multicolumn{1}{c|}{-} & 0.264 & \multicolumn{1}{c}{-}  \\
    
    & \scalea{MT-F} 
    & 0.390 & \reduce{14.10\% $\downarrow$} & 0.410 & \reduce{7.03\% $\downarrow$}
    & 0.201 & \reduce{14.47\% $\downarrow$} & 0.247 & \reduce{6.44\% $\downarrow$}  \\
    
    & \scalea{MoLA} 
    & 0.379 & \reduce{16.52\% $\downarrow$} & 0.400 & \reduce{9.30\% $\downarrow$}
    & 0.173 & \reduce{26.38\% $\downarrow$} & 0.211 & \reduce{20.08\% $\downarrow$}  \\
    \hline
    
    \multirow{3}{*}{{\rotatebox{90}{\scalebox{0.95}{192}}}} 
    & \scalea{AR-F} 
    & 0.525 & \multicolumn{1}{c|}{-} & 0.484 & \multicolumn{1}{c|}{-} 
    & 0.272 & \multicolumn{1}{c|}{-} & 0.293 & \multicolumn{1}{c}{-}  \\
    
    & \scalea{MT-F} 
    & 0.443 & \reduce{15.62\% $\downarrow$} & 0.441 & \reduce{8.88\% $\downarrow$}
    & 0.250 & \reduce{8.09\% $\downarrow$} & 0.283 & \reduce{3.41\% $\downarrow$}  \\
    
    & \scalea{MoLA} 
    & 0.436 & \reduce{16.95\% $\downarrow$} & 0.432 & \reduce{10.74\% $\downarrow$}
    & 0.246 & \reduce{9.56\% $\downarrow$} & 0.280 & \reduce{4.44\% $\downarrow$}  \\
    \hline
    
    \multirow{3}{*}{{\rotatebox{90}{\scalebox{0.95}{336}}}} 
    & \scalea{AR-F} 
    & 0.580 & \multicolumn{1}{c|}{-} & 0.518 & \multicolumn{1}{c|}{-} 
    & 0.368 & \multicolumn{1}{c|}{-} & 0.357 & \multicolumn{1}{c}{-}  \\
    
    & \scalea{MT-F} 
    & 0.480 & \reduce{17.24\% $\downarrow$} & 0.457 & \reduce{11.78\% $\downarrow$}
    & 0.302 & \reduce{17.93\% $\downarrow$} & 0.317 & \reduce{11.20\% $\downarrow$}  \\
    
    & \scalea{MoLA} 
    & 0.477 & \reduce{17.76\% $\downarrow$} & 0.456 & \reduce{11.97\% $\downarrow$}
    & 0.277 & \reduce{24.73\% $\downarrow$} & 0.296 & \reduce{17.09\% $\downarrow$}  \\
    \hline
    
    \multirow{3}{*}{{\rotatebox{90}{\scalebox{0.95}{720}}}} 
    & \scalea{AR-F} 
    & 0.620 & \multicolumn{1}{c|}{-} & 0.550 & \multicolumn{1}{c|}{-} 
    & 0.600 & \multicolumn{1}{c|}{-} & 0.444 & \multicolumn{1}{c}{-}  \\
    
    & \scalea{MT-F} 
    & 0.484 & \reduce{21.94\% $\downarrow$} & 0.479 & \reduce{12.91\% $\downarrow$}
    & 0.370 & \reduce{38.33\% $\downarrow$} & 0.362 & \reduce{18.47\% $\downarrow$}  \\
    
    & \scalea{MoLA} 
    & 0.480 & \reduce{22.58\% $\downarrow$} & 0.478 & \reduce{13.09\% $\downarrow$}
    & 0.367 & \reduce{38.83\% $\downarrow$} & 0.356 & \reduce{19.82\% $\downarrow$}  \\
    \hline
\end{tabular}
\begin{tablenotes}
    \item  \scriptsize \textit{Note}: $\Delta$ denotes the relative error improvement compared to iTransformer with AR-F paradigm.
\end{tablenotes}
\end{threeparttable}

\vspace{-5mm}

\end{wraptable}

This section presents a comparative analysis of the MoLA paradigm against the traditional AR-F and MT-F paradigms, using iTransformer as the base model. The evaluation covers two datasets (ETTh1 and Weather) and four different forecasting horizons, with the results summarized in Table~\ref{tab:compara}.

Across all horizons, MoLA consistently outperforms both AR-F and MT-F paradigms. While MT-F reduces error accumulation, which is a known issue with AR-F, MoLA goes a step further by mitigating the expressiveness bottleneck of MT-F. By introducing modularized and personalized fine-tuning for each forecasting horizon, MoLA enhances the forecast accuracy, particularly for longer forecasting horizons where performance gains are more pronounced.
These findings demonstrate that MoLA is more than just an incremental improvement over MT-F, it represents a meaningful advancement in fine-tuning techniques for TSF to overcome the expressiveness bottleneck.

\subsection{Generalization Studies}
In this section, we explore the generality of the MoLA paradigm to enhance varying forecasting models and encompass existing fine-tuning techniques.

\paragraph{Generalization to Fine-tuning Techniques.}
\label{subsec:gene_module}
\begin{table*}
\caption{Generalization results.}
\label{tab:finetune_transposed}
\renewcommand{\arraystretch}{1}
\centering
\tiny
\begin{threeparttable}
\setlength{\tabcolsep}{1.4pt}
\begin{tabular}{c|c|cccc|crcr|crcr|crcr|crcr}
    \toprule
    \multicolumn{2}{l}{Variants}
    & \multicolumn{4}{c}{MT-F} 
    & \multicolumn{4}{c}{MoLA-Ada} 
    & \multicolumn{4}{c}{MoLA-IA$^3$}
    & \multicolumn{4}{c}{MoLA-R}
    & \multicolumn{4}{c}{MoLA} \\
    \cmidrule(lr){3-6} \cmidrule(lr){7-10} \cmidrule(lr){11-14} \cmidrule(lr){15-18} \cmidrule(lr){19-22} 

    \multicolumn{2}{l}{Metrics} & MSE & \multicolumn{1}{c}{$\Delta$} & MAE & \multicolumn{1}{c|}{$\Delta$} 
    & MSE & \multicolumn{1}{c}{$\Delta$} & MAE & \multicolumn{1}{c|}{$\Delta$}
    & MSE & \multicolumn{1}{c}{$\Delta$} & MAE & \multicolumn{1}{c|}{$\Delta$}
    & MSE & \multicolumn{1}{c}{$\Delta$} & MAE & \multicolumn{1}{c|}{$\Delta$}
    & MSE & \multicolumn{1}{c}{$\Delta$} & MAE & \multicolumn{1}{c}{$\Delta$} \\
    \midrule
    
    \multirow{5}{*}{{\rotatebox{90}{\scalebox{0.95}{ETTh2}}}} & 96
    & 0.301 & - & 0.349 & - 
    & 0.292 & \reduce{2.93\%$\downarrow$} & 0.342 & \reduce{2.10\%$\downarrow$}
    & 0.304 & \reduce{0.88\%$\uparrow$} & 0.346 & \reduce{0.84\%$\downarrow$}
    & 0.293 & \reduce{2.53\%$\downarrow$} & 0.342 & \reduce{2.13\%$\downarrow$}
    & 0.293 & \reduce{2.57\%$\downarrow$} & 0.341 & \reduce{2.25\%$\downarrow$} \\
    
    & 192
    & 0.382 & - & 0.402 & -
    & 0.384 & \reduce{0.55\%$\uparrow$} & 0.398 & \reduce{1.08\%$\downarrow$}
    & 0.381 & \reduce{0.17\%$\downarrow$} & 0.397 & \reduce{1.34\%$\downarrow$}
    & 0.380 & \reduce{0.54\%$\downarrow$} & 0.398 & \reduce{1.00\%$\downarrow$}
    & 0.377 & \reduce{1.25\%$\downarrow$} & 0.397 & \reduce{1.32\%$\downarrow$} \\
    
    & 336
    & 0.430 & - & 0.434 & - 
    & 0.420 & \reduce{2.31\%$\downarrow$} & 0.431 & \reduce{0.78\%$\downarrow$}
    & 0.418 & \reduce{2.88\%$\downarrow$} & 0.428 & \reduce{1.48\%$\downarrow$}
    & 0.422 & \reduce{1.81\%$\downarrow$} & 0.433 & \reduce{0.23\%$\downarrow$}
    & 0.421 & \reduce{2.06\%$\downarrow$} & 0.429 & \reduce{1.22\%$\downarrow$} \\
    
    & 720
    & 0.447 & - & 0.455 & -
    & 0.430 & \reduce{3.87\%$\downarrow$} & 0.446 & \reduce{1.97\%$\downarrow$}
    & 0.437 & \reduce{2.23\%$\downarrow$} & 0.447 & \reduce{1.68\%$\downarrow$}
    & 0.428 & \reduce{4.25\%$\downarrow$} & 0.444 & \reduce{2.44\%$\downarrow$}
    & 0.418 & \reduce{6.49\%$\downarrow$} & 0.438 & \reduce{3.67\%$\downarrow$} \\
    \cmidrule(lr){2-22}

    & Avg
    & 0.390 & - & 0.410 & -
    & 0.382 & \reduce{2.05\% $\downarrow$} & 0.402 & \reduce{1.95\% $\downarrow$}
    & 0.385 & \reduce{1.28\% $\downarrow$} & 0.405 & \reduce{1.22\% $\downarrow$}
    & 0.381 & \reduce{2.31\% $\downarrow$} & 0.402 & \reduce{1.98\% $\downarrow$}
    & 0.377 & \reduce{3.33\% $\downarrow$} & 0.401 & \reduce{2.20\% $\downarrow$} \\
    \midrule
    
    \multirow{5}{*}{{\rotatebox{90}{\scalebox{0.95}{Weather}}}} & 96
    & 0.201 & - & 0.247 & -
    & 0.202 & \reduce{0.43\%$\uparrow$} & 0.246 & \reduce{0.39\%$\downarrow$}
    & 0.204 & \reduce{1.69\%$\uparrow$} & 0.249 & \reduce{0.80\%$\uparrow$}
    & 0.202 & \reduce{0.40\%$\uparrow$} & 0.245 & \reduce{1.00\%$\downarrow$}
    & 0.173 & \reduce{13.78\%$\downarrow$} & 0.211 & \reduce{14.71\%$\downarrow$} \\
    
    & 192
    & 0.250 & - & 0.283 & -
    & 0.248 & \reduce{0.83\%$\downarrow$} & 0.279 & \reduce{1.23\%$\downarrow$}
    & 0.250 & \reduce{0.14\%$\downarrow$} & 0.281 & \reduce{0.63\%$\downarrow$}
    & 0.248 & \reduce{0.90\%$\downarrow$} & 0.279 & \reduce{1.45\%$\downarrow$}
    & 0.246 & \reduce{1.43\%$\downarrow$} & 0.280 & \reduce{1.02\%$\downarrow$} \\
    
    & 336
    & 0.302 & - & 0.317 & -
    & 0.299 & \reduce{0.89\%$\downarrow$} & 0.317 & \reduce{0.03\%$\downarrow$}
    & 0.300 & \reduce{0.60\%$\downarrow$} & 0.315 & \reduce{0.66\%$\downarrow$}
    & 0.280 & \reduce{7.28\%$\downarrow$} & 0.297 & \reduce{6.44\%$\downarrow$}
    & 0.277 & \reduce{8.28\%$\downarrow$} & 0.296 & \reduce{6.63\%$\downarrow$} \\
    
    & 720
    & 0.370 & - & 0.362 & -
    & 0.361 & \reduce{2.41\%$\downarrow$} & 0.351 & \reduce{3.04\%$\downarrow$}
    & 0.365 & \reduce{1.40\%$\downarrow$} & 0.356 & \reduce{1.66\%$\downarrow$}
    & 0.357 & \reduce{3.57\%$\downarrow$} & 0.349 & \reduce{3.63\%$\downarrow$}
    & 0.367 & \reduce{0.86\%$\downarrow$} & 0.356 & \reduce{1.60\%$\downarrow$} \\
    \cmidrule(lr){2-22}
    
    & Avg
    & 0.281 & - & 0.302 & -
    & 0.278 & \reduce{1.07\% $\downarrow$} & 0.299 & \reduce{0.99\% $\downarrow$}
    & 0.280 & \reduce{0.36\% $\downarrow$} & 0.300 & \reduce{0.66\% $\downarrow$}
    & 0.272 & \reduce{3.20\% $\downarrow$} & 0.293 & \reduce{2.98\% $\downarrow$}
    & 0.266 & \reduce{5.34\% $\downarrow$} & 0.286 & \reduce{5.30\% $\downarrow$} \\
    \bottomrule
\end{tabular}
\begin{tablenotes}
    \item  \scriptsize \textit{Note}: $\Delta$ denotes the relative error improvement compared to iTransformer with MT-F paradigm.
\end{tablenotes}
\end{threeparttable}
\end{table*}

We implement MoLA by replacing the LoRA modules with two alternative fine-tuning techniques: Adapter~\citep{Adapter} (MoLA-Ada) and IA$^3$~\citep{IA3} (MoLA-IA$^3$), to show its support of existing fine-tuning techniques. Both Adapter and IA$^3$ are well-established parameter-efficient fine-tuning technologies. We also introduce a variant that applies step-specific modularized fine-tuning with standard LoRA modules (MoLA-R) for comparison. Detailed illustrations of their technical differences and parameter settings are provided in Appendix E.3.
The results in Table~\ref{tab:finetune_transposed} indicate that the devised variants exhibit comparable improvements over the MT-F paradigm, affirming MoLA's flexibility in integrating diverse fine-tuning strategies. 
The standard MoLA, enabling parameter sharing across different tasks with the MoE mechanism, outperforms the variants since it leverages the fruitful inter-step information while maintaining model efficiency.

\paragraph{Generalization to Forecasting Models.}
\label{subsec:gene_model}
We incorporate MoLA into several well-established TSF models: iTransformer, Autoformer, Informer, and Transformer. The results, averaged across different prediction lengths (96, 192, 336, 720) and accompanied by 95\% confidence intervals, are presented in Figure~\ref{fig:backbone}. 
Overall, MoLA improves the performance of these models over the traditional MT-F paradigm. And notably, Autoformer and Informer benefit significantly from the incorporation of MoLA, showing a relative MSE improvement of over 4\% on both the ECL and Weather datasets. These results demonstrate the generality and broad applicability of MoLA, reinforcing its potential as a plug-and-play strategy for enhancing various neural forecasting models in TSF.

\begin{figure}
\begin{center}
\subfigure[ECL with MSE.]{\includegraphics[width=0.24\linewidth]{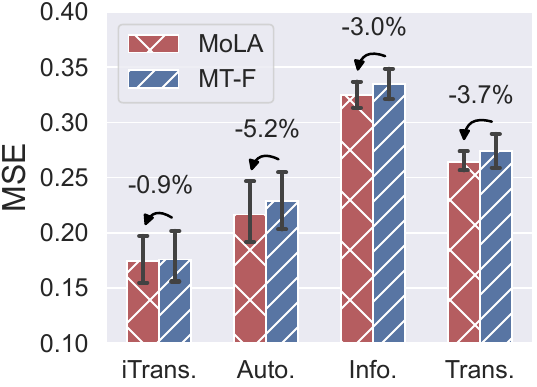}}
\hfill
\subfigure[ECL with MAE]{\includegraphics[width=0.24\linewidth]{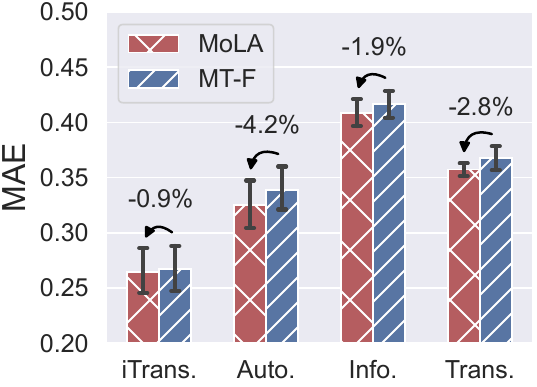}}
\hfill
\subfigure[Weather with MSE.]{\includegraphics[width=0.24\linewidth]{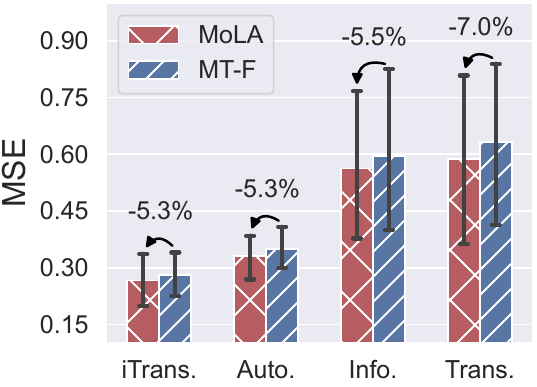}}
\hfill
\subfigure[Weather with MAE.]{\includegraphics[width=0.24\linewidth]{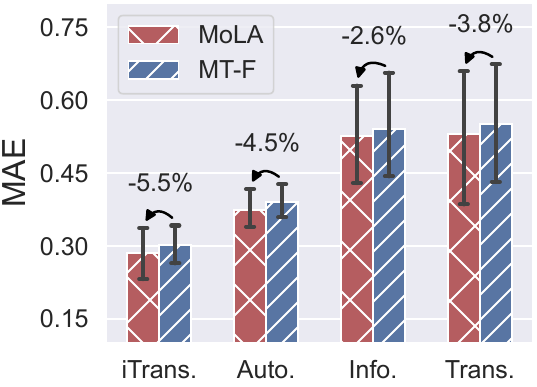}}
\caption{Benefit of incorporating MoLA in varying models, shown with colored bars for means over forecasting lengths (96, 192, 336, 720) and error bars for 95\% confidence intervals. }
\label{fig:backbone}
\end{center}
\vspace{-5mm}
\end{figure}

\subsection{Sensitivity Studies}
\label{subsec:sensi}

\paragraph{Sensitivity on Segment Number.}

\begin{wraptable}{r}{0.5\linewidth}
\vspace{-4mm}
\centering
\caption{Results of varying segment number.}
\label{tab:seg-boost}
\tiny
\tabcolsep=2.2pt
\renewcommand{\arraystretch}{1.2} 
\begin{threeparttable}
\begin{tabular}{l|cr|cr|cr|cr}
    \hline
    & \multicolumn{4}{c|}{ETTh1} & \multicolumn{4}{c}{Weather} \\
    \hline
    $\K$ & MSE & \multicolumn{1}{c|}{$\Delta$} & MAE & \multicolumn{1}{c|}{$\Delta$} & MSE & \multicolumn{1}{c|}{$\Delta$} & MAE & \multicolumn{1}{c}{$\Delta$} \\
    \hline
    MT-F & 0.390 & \multicolumn{1}{c|}{-} & 0.410 & \multicolumn{1}{c|}{-} & 0.201 & \multicolumn{1}{c|}{-} & 0.247 & \multicolumn{1}{c}{-} \\
    
    1 & 0.390 & \reduce{0.09\% $\downarrow$} & 0.410 & \reduce{0.21\% $\downarrow$} & 0.201 & \reduce{0.02\% $\downarrow$} & 0.246 & \reduce{0.69\% $\downarrow$} \\

    2 & 0.383 & \reduce{2.02\% $\downarrow$} & 0.403 & \reduce{1.79\% $\downarrow$} & 0.202 & \reduce{0.25\% $\uparrow$} & 0.246 & \reduce{0.61\% $\downarrow$} \\

    3 & 0.389 & \reduce{0.33\% $\downarrow$} & 0.409 & \reduce{0.24\% $\downarrow$} & 0.200 & \reduce{0.54\% $\downarrow$} & 0.244 & \reduce{1.23\% $\downarrow$} \\
    
    4 & 0.379 & \reduce{2.93\% $\downarrow$} & 0.400 & \reduce{2.47\% $\downarrow$} & 0.201 & \reduce{0.03\% $\downarrow$} & 0.245 & \reduce{1.07\% $\downarrow$} \\
    
    8 & 0.391 & \reduce{0.14\% $\uparrow$} & 0.408 & \reduce{0.70\% $\downarrow$} & 0.173 & \reduce{13.99\% $\downarrow$} & 0.211 & \reduce{14.68\% $\downarrow$} \\
    
    16 & 0.390 & \reduce{0.10\% $\downarrow$} & 0.407 & \reduce{0.79\% $\downarrow$} & 0.176 & \reduce{12.57\% $\downarrow$} & 0.214 & \reduce{13.53\% $\downarrow$} \\
    
    32 & 0.388 & \reduce{0.66\% $\downarrow$} & 0.403 & \reduce{1.81\% $\downarrow$} & 0.174 & \reduce{13.43\% $\downarrow$} & 0.213 & \reduce{14.02\% $\downarrow$} \\
    
    48 & 0.389 & \reduce{0.27\% $\downarrow$} & 0.405 & \reduce{1.37\% $\downarrow$} & 0.198 & \reduce{1.50\% $\downarrow$} & 0.244 & \reduce{1.47\% $\downarrow$} \\
    
    96 & 0.389 & \reduce{0.31\% $\downarrow$} & 0.404 & \reduce{1.61\% $\downarrow$} & 0.181 & \reduce{9.92\% $\downarrow$} & 0.221 & \reduce{10.84\% $\downarrow$} \\
    \hline
\end{tabular}
\begin{tablenotes}
\scriptsize
\item \textit{Note}: $\Delta$ denotes the relative error improvement compared to iTransformer with MT-F paradigm. The results are generated with forecasting length fixed at 96.
\end{tablenotes}
\end{threeparttable}
\vspace{-1mm}
\end{wraptable}

In this section, we investigate the impact of the segment number $\K$ on the performance of MoLA.
We conduct experiments on the ETTh1 and Weather datasets by varying $\K$ while keeping $\T=96$, and the results are presented in Table~\ref{tab:seg-boost}.
First, MoLA outperforms MT-F across nearly all values of $\K$. Secondly, smaller segment sizes (\eg $\K=4$ for ETTh1 and $\K=8$ for the Weather dataset) achieve the best performance, demonstrating the effectiveness of the segmentation strategy in enhancing MoLA’s forecasting accuracy since it enables the utilization of correlations in neighboring steps.
Therefore, although ideally each prediction length would have its own set of LoRA weights, \ie $\K=\T$, the segmentation strategy offers a flexible solution that reduces computational complexity while simultaneously improving forecasting performance.

\paragraph{Sensitivity on Other Hyperparameters.}  
We then examine the sensitivity of MoLA paradigm to three key hyperparameters: the rank $r$ of the LoRA expert modules, the learning rate $\eta$ used during fine-tuning and the number of LoRA expert modules. Experiments on ETTh1 and Weather datasets (prediction lengths 192 and 336) reveal distinct patterns, as shown in Figure~\ref{fig:sensi}. First, for rank $r$ (tested over $\{4, 8, 16, 32, 64\}$), both excessively small and large ranks degrade performance, with the optimal value balancing flexibility and overfitting risks. For instance, on ETTh1, $r=8$ achieves minimal MSE values of 0.441 (192) and 0.484 (336), while higher ranks overfit and lower ranks lack expressiveness. Second, learning rate $\eta$ exhibits sharper sensitivity: increasing $\eta$ from 0.0002 to 0.001 on ETTh1 reduces MSE from 0.459 to 0.440 for 192-step forecasts, aligning with prior observations on its critical role in convergence~\citep{LoRA-sensi}. Finally, the number of LoRA experts shows stable performance unless it severely mismatches the segment count $\K$. For example, on Weather (192-step), expanding experts from 2 to 10 only reduces MSE by 0.001, suggesting robustness to moderate variations. These trends highlight the necessity of calibrating $r$ and $\eta$ precisely, while the expert count offers flexibility within practical bounds. Extended results across forecasting lengths are provided in Appendix E.4.

\begin{figure}[t]
\begin{center}
\subfigure[Varying $r$ on ETTh1.]{\includegraphics[width=0.28\linewidth]{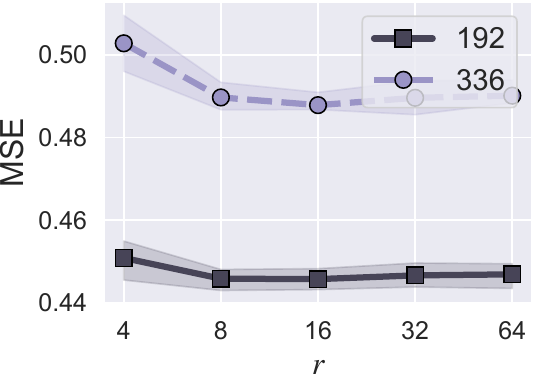}\label{subfig:sensi-a}}
\hspace{8pt}
\subfigure[Varying $\eta$ on ETTh1.]{\includegraphics[width=0.28\linewidth]{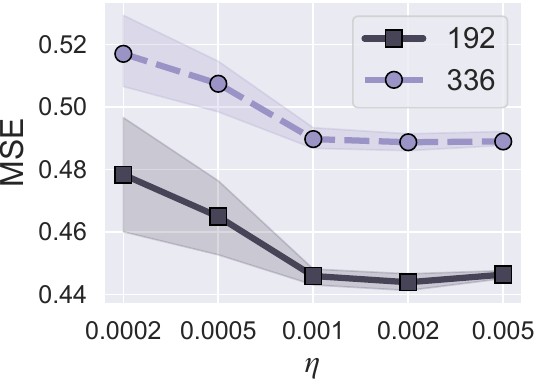}\label{subfig:sensi-c}}
\hspace{8pt}
\subfigure[Varying $\mP$ on ETTh1.]{\includegraphics[width=0.28\linewidth]{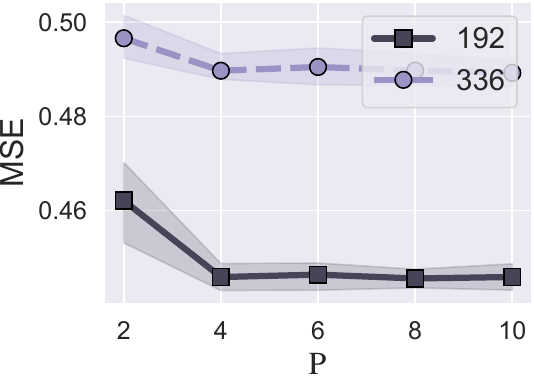}\label{subfig:sensi-e}}

\subfigure[Varying $r$ on Weather.]{\includegraphics[width=0.28\linewidth]{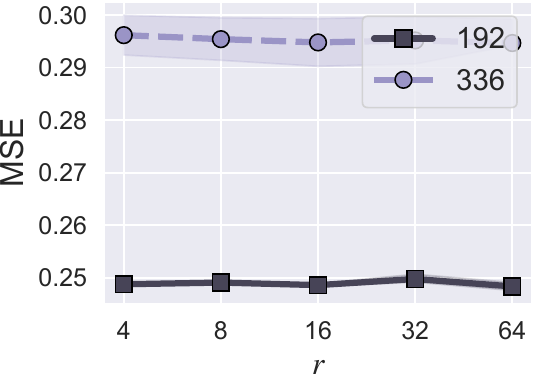}\label{subfig:sensi-b}}
\hspace{8pt}
\subfigure[Varying $\eta$ on Weather.]{\includegraphics[width=0.28\linewidth]{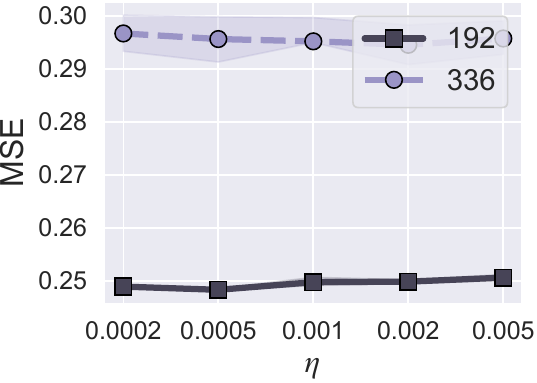}\label{subfig:sensi-d}}
\hspace{8pt}
\subfigure[Varying $\mP$ on Weather.]{\includegraphics[width=0.28\linewidth]{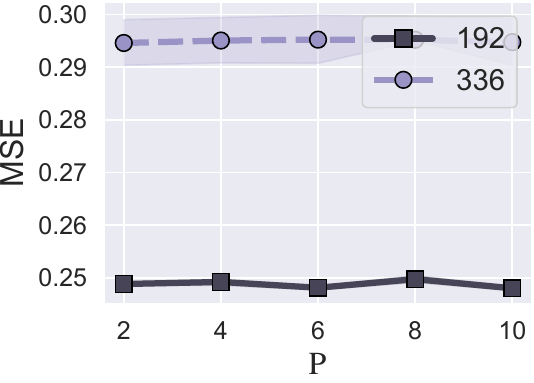}\label{subfig:sensi-f}}
\caption{Performance given varying rank $r$, learning rate $\eta$ and the number of experts $\mP$.}
\label{fig:sensi}
\end{center}
\end{figure}

\section{Conclusion}\label{sec:conclusion}
In this work, we address the expressiveness bottleneck in time-series forecasting, a limitation that hinders accurate multi-step prediction with existing multi-task paradigm. We identify that conventional linear decoders, by sharing representations across different forecast steps, are inherently inadequate for capturing step-specific dynamics, resulting in unavoidable forecast errors. 
To handle this challenge, we propose a two-stage framework: first, pre-training a foundation model for one-step-ahead prediction, and then adapting it for multi-step forecasting using step-specific Low-Rank Adaptation (LoRA) modules. Building on this design, we further introduce the Mixture-of-LoRA (MoLA) model, which leverages segment-based adaptation and adaptively weighted LoRA experts to achieve partial parameter sharing across forecast steps. This approach enables MoLA to flexibly balance expressiveness and efficiency, while effectively exploiting interdependencies between forecast steps. 
Extensive experiments across diverse datasets and backbone models demonstrate that MoLA consistently improves forecasting accuracy and outperforms state-of-the-art time-series forecasting methods, highlighting the promise of modularized fine-tuning and partial parameter sharing for multi-horizon forecasting.

\textbf{Limitations and Future Work.} This study primarily focuses on adaptation using LoRA expert modules for fine-tuning the $\mS$-step foundation model. However, the parameter cost of LoRA still scales with model size and forecast horizon. Exploring alternative parameter-efficient adaptation techniques, such as prompt-tuning or adapter-tuning, may further enhance the scalability and applicability of the proposed framework. Additionally, in this work, the segmentation of forecast steps is manually specified. Automatically determining segment boundaries based on data-driven analysis of time-series patterns and periodicity represents a promising direction for further optimizing the MoLA framework.

{
\small

\bibliographystyle{plain}
\bibliography{refs}

\begin{thebibliography}{10}

\bibitem{weather_app2}
Anna Allen, Stratis Markou, Will Tebbutt, James Requeima, Wessel~P Bruinsma, Tom~R Andersson, Michael Herzog, Nicholas~D Lane, Matthew Chantry, J~Scott Hosking, et~al.
\newblock End-to-end data-driven weather prediction.
\newblock {\em Nature}, pages 1--3, 2025.

\bibitem{tcn}
Shaojie Bai, J~Zico Kolter, and Vladlen Koltun.
\newblock An empirical evaluation of generic convolutional and recurrent networks for sequence modeling.
\newblock {\em arXiv preprint arXiv:1803.01271}, 2018.

\bibitem{LoRA-sensi}
Dan Biderman, Jose~Gonzalez Ortiz, Jacob Portes, Mansheej Paul, Philip Greengard, Connor Jennings, Daniel King, Sam Havens, Vitaliy Chiley, Jonathan Frankle, et~al.
\newblock Lora learns less and forgets less.
\newblock {\em arXiv preprint arXiv:2405.09673}, 2024.

\bibitem{box2015time}
George~EP Box, Gwilym~M Jenkins, Gregory~C Reinsel, and Greta~M Ljung.
\newblock {\em Time series analysis: forecasting and control}.
\newblock John Wiley \& Sons, 2015.

\bibitem{stemgnn}
Defu Cao, Yujing Wang, Juanyong Duan, Ce~Zhang, Xia Zhu, Congrui Huang, Yunhai Tong, Bixiong Xu, Jing Bai, Jie Tong, et~al.
\newblock Spectral temporal graph neural network for multivariate time-series forecasting.
\newblock In {\em NeurIPS}, volume~33, pages 17766--17778, 2020.

\bibitem{TSF-FT2}
Ching Chang, Wen-Chih Peng, and Tien-Fu Chen.
\newblock Llm4ts: Two-stage fine-tuning for time-series forecasting with pre-trained llms.
\newblock {\em arXiv preprint arXiv:2308.08469}, 2023.

\bibitem{magnn}
Ling Chen, Donghui Chen, Zongjiang Shang, Binqing Wu, Cen Zheng, Bo~Wen, and Wei Zhang.
\newblock Multi-scale adaptive graph neural network for multivariate time series forecasting.
\newblock {\em IEEE Transactions on Knowledge and Data Engineering}, 35(10):10748--10761, 2023.

\bibitem{qlora}
Tim Dettmers, Artidoro Pagnoni, Ari Holtzman, and Luke Zettlemoyer.
\newblock Qlora: Efficient finetuning of quantized llms.
\newblock {\em Advances in neural information processing systems}, 36:10088--10115, 2023.

\bibitem{sora}
Ning Ding, Xingtai Lv, Qiaosen Wang, Yulin Chen, Bowen Zhou, Zhiyuan Liu, and Maosong Sun.
\newblock Sparse low-rank adaptation of pre-trained language models.
\newblock {\em arXiv preprint arXiv:2311.11696}, 2023.

\bibitem{TSMixer2}
Vijay Ekambaram, Arindam Jati, Nam Nguyen, Phanwadee Sinthong, and Jayant Kalagnanam.
\newblock Tsmixer: Lightweight mlp-mixer model for multivariate time series forecasting.
\newblock In {\em SIGKDD}, page 459–469, 2023.

\bibitem{note_on_lora}
Vlad Fomenko, Han Yu, Jongho Lee, Stanley Hsieh, and Weizhu Chen.
\newblock A note on lora.
\newblock {\em arXiv preprint arXiv:2404.05086}, 2024.

\bibitem{energy_app}
Alberto Gasparin, Slobodan Lukovic, and Cesare Alippi.
\newblock Deep learning for time series forecasting: The electric load case.
\newblock {\em CAAI Transactions on Intelligence Technology}, 7(1):1--25, 2022.

\bibitem{cross_attn}
Mozhdeh Gheini, Xiang Ren, and Jonathan May.
\newblock Cross-attention is all you need: Adapting pretrained transformers for machine translation.
\newblock {\em arXiv preprint arXiv:2104.08771}, 2021.

\bibitem{ghimire2024two}
Sujan Ghimire, Ravinesh~C Deo, David Casillas-P{\'e}rez, and Sancho Salcedo-Sanz.
\newblock Two-step deep learning framework with error compensation technique for short-term, half-hourly electricity price forecasting.
\newblock {\em Applied Energy}, 353:122059, 2024.

\bibitem{TSF-FT4}
Divij Gupta, Anubhav Bhatti, Suraj Parmar, Chen Dan, Yuwei Liu, Bingjie Shen, and San Lee.
\newblock Low-rank adaptation of time series foundational models for out-of-domain modality forecasting.
\newblock In {\em Proceedings of the 26th International Conference on Multimodal Interaction}, pages 382--386, 2024.

\bibitem{SPT}
Haoyu He, Jianfei Cai, Jing Zhang, Dacheng Tao, and Bohan Zhuang.
\newblock Sensitivity-aware visual parameter-efficient fine-tuning.
\newblock In {\em Proceedings of the IEEE/CVF International Conference on Computer Vision}, pages 11825--11835, 2023.

\bibitem{parallel_adapter}
Junxian He, Chunting Zhou, Xuezhe Ma, Taylor Berg-Kirkpatrick, and Graham Neubig.
\newblock Towards a unified view of parameter-efficient transfer learning.
\newblock {\em arXiv preprint arXiv:2110.04366}, 2021.

\bibitem{mera}
Shwai He, Run-Ze Fan, Liang Ding, Li~Shen, Tianyi Zhou, and Dacheng Tao.
\newblock Mera: Merging pretrained adapters for few-shot learning.
\newblock {\em arXiv preprint arXiv:2308.15982}, 2023.

\bibitem{Adapter}
Neil Houlsby, Andrei Giurgiu, Stanislaw Jastrzebski, Bruna Morrone, Quentin De~Laroussilhe, Andrea Gesmundo, Mona Attariyan, and Sylvain Gelly.
\newblock Parameter-efficient transfer learning for nlp.
\newblock In {\em International conference on machine learning}, pages 2790--2799. PMLR, 2019.

\bibitem{serial-adapter}
Neil Houlsby, Andrei Giurgiu, Stanislaw Jastrzebski, Bruna Morrone, Quentin De~Laroussilhe, Andrea Gesmundo, Mona Attariyan, and Sylvain Gelly.
\newblock Parameter-efficient transfer learning for nlp.
\newblock In {\em International conference on machine learning}, pages 2790--2799. PMLR, 2019.

\bibitem{LoRA}
Edward~J Hu, Yelong Shen, Phillip Wallis, Zeyuan Allen-Zhu, Yuanzhi Li, Shean Wang, Lu~Wang, and Weizhu Chen.
\newblock Lora: Low-rank adaptation of large language models.
\newblock {\em arXiv preprint arXiv:2106.09685}, 2021.

\bibitem{fintsb}
Yifan Hu, Yuante Li, Peiyuan Liu, Yuxia Zhu, Naiqi Li, Tao Dai, Shu-tao Xia, Dawei Cheng, and Changjun Jiang.
\newblock Fintsb: A comprehensive and practical benchmark for financial time series forecasting.
\newblock {\em arXiv preprint arXiv:2502.18834}, 2025.

\bibitem{TSF-FT1}
Subina Khanal, Seshu Tirupathi, Giulio Zizzo, Ambrish Rawat, and Torben~Bach Pedersen.
\newblock Domain adaptation for time series transformers using one-step fine-tuning.
\newblock {\em arXiv preprint arXiv:2401.06524}, 2024.

\bibitem{Adam}
Diederik~P. Kingma and Jimmy Ba.
\newblock Adam: {A} method for stochastic optimization.
\newblock In {\em ICLR}, 2015.

\bibitem{CoDA}
Tao Lei, Junwen Bai, Siddhartha Brahma, Joshua Ainslie, Kenton Lee, Yanqi Zhou, Nan Du, Vincent Zhao, Yuexin Wu, Bo~Li, et~al.
\newblock Conditional adapters: Parameter-efficient transfer learning with fast inference.
\newblock {\em Advances in Neural Information Processing Systems}, 36:8152--8172, 2023.

\bibitem{Informer}
Jianxin Li, Xiong Hui, and Wancai Zhang.
\newblock Informer: Beyond efficient transformer for long sequence time-series forecasting.
\newblock In {\em AAAI}, 2021.

\bibitem{loftq}
Yixiao Li, Yifan Yu, Chen Liang, Pengcheng He, Nikos Karampatziakis, Weizhu Chen, and Tuo Zhao.
\newblock Loftq: Lora-fine-tuning-aware quantization for large language models.
\newblock {\em arXiv preprint arXiv:2310.08659}, 2023.

\bibitem{PaFi}
Baohao Liao, Yan Meng, and Christof Monz.
\newblock Parameter-efficient fine-tuning without introducing new latency.
\newblock {\em arXiv preprint arXiv:2305.16742}, 2023.

\bibitem{sparsetsf}
Shengsheng Lin, Weiwei Lin, Wentai Wu, Haojun Chen, and Junjie Yang.
\newblock Sparsetsf: Modeling long-term time series forecasting with 1k parameters.
\newblock {\em arXiv preprint arXiv:2405.00946}, 2024.

\bibitem{segrnn}
Shengsheng Lin, Weiwei Lin, Wentai Wu, Feiyu Zhao, Ruichao Mo, and Haotong Zhang.
\newblock Segrnn: Segment recurrent neural network for long-term time series forecasting.
\newblock {\em arXiv preprint arXiv:2308.11200}, 2023.

\bibitem{IA3}
Haokun Liu, Derek Tam, Mohammed Muqeeth, Jay Mohta, Tenghao Huang, Mohit Bansal, and Colin~A Raffel.
\newblock Few-shot parameter-efficient fine-tuning is better and cheaper than in-context learning.
\newblock {\em Advances in Neural Information Processing Systems}, 35:1950--1965, 2022.

\bibitem{SCINet}
Minhao Liu, Ailing Zeng, Muxi Chen, Zhijian Xu, Qiuxia Lai, Lingna Ma, and Qiang Xu.
\newblock Scinet: time series modeling and forecasting with sample convolution and interaction.
\newblock In {\em NeurIPS}, 2022.

\bibitem{itransformer}
Yong Liu, Tengge Hu, Haoran Zhang, Haixu Wu, Shiyu Wang, Lintao Ma, and Mingsheng Long.
\newblock itransformer: Inverted transformers are effective for time series forecasting.
\newblock In {\em ICLR}, 2024.

\bibitem{nguyen2024scaling}
Tung Nguyen, Rohan Shah, Hritik Bansal, Troy Arcomano, Romit Maulik, Rao Kotamarthi, Ian Foster, Sandeep Madireddy, and Aditya Grover.
\newblock Scaling transformer neural networks for skillful and reliable medium-range weather forecasting.
\newblock {\em Advances in Neural Information Processing Systems}, 37:68740--68771, 2024.

\bibitem{TSF-FT5}
Tong Nie, Yuewen Mei, Guoyang Qin, Jian Sun, and Wei Ma.
\newblock Channel-aware low-rank adaptation in time series forecasting.
\newblock In {\em Proceedings of the 33rd ACM International Conference on Information and Knowledge Management}, pages 3959--3963, 2024.

\bibitem{PatchTST}
Yuqi Nie, Nam~H Nguyen, Phanwadee Sinthong, and Jayant Kalagnanam.
\newblock A time series is worth 64 words: Long-term forecasting with transformers.
\newblock In {\em ICLR}, 2023.

\bibitem{fredformer}
Xihao Piao, Zheng Chen, Taichi Murayama, Yasuko Matsubara, and Yasushi Sakurai.
\newblock Fredformer: Frequency debiased transformer for time series forecasting.
\newblock In {\em Proceedings of the 30th ACM SIGKDD Conference on Knowledge Discovery and Data Mining}, pages 2400--2410, 2024.

\bibitem{deepar}
David Salinas, Valentin Flunkert, Jan Gasthaus, and Tim Januschowski.
\newblock Deepar: Probabilistic forecasting with autoregressive recurrent networks.
\newblock {\em Int. J. Forecast}, 36(3):1181--1191, 2020.

\bibitem{Transformer}
Ashish Vaswani, Noam Shazeer, Niki Parmar, Jakob Uszkoreit, Llion Jones, Aidan~N Gomez, Lukasz Kaiser, and Illia Polosukhin.
\newblock Attention is all you need.
\newblock In {\em NeurIPS}, 2017.

\bibitem{FAR}
Danilo Vucetic, Mohammadreza Tayaranian, Maryam Ziaeefard, James~J Clark, Brett~H Meyer, and Warren~J Gross.
\newblock Efficient fine-tuning of bert models on the edge.
\newblock In {\em 2022 IEEE International Symposium on Circuits and Systems (ISCAS)}, pages 1838--1842. IEEE, 2022.

\bibitem{micn}
Huiqiang Wang, Jian Peng, Feihu Huang, Jince Wang, Junhui Chen, and Yifei Xiao.
\newblock Micn: Multi-scale local and global context modeling for long-term series forecasting.
\newblock In {\em The eleventh international conference on learning representations}, 2023.

\bibitem{Timesnet}
Haixu Wu, Tengge Hu, Yong Liu, Hang Zhou, Jianmin Wang, and Mingsheng Long.
\newblock Timesnet: Temporal 2d-variation modeling for general time series analysis.
\newblock In {\em ICLR}, 2023.

\bibitem{Autoformer}
Haixu Wu, Jiehui Xu, Jianmin Wang, and Mingsheng Long.
\newblock Autoformer: Decomposition transformers with {Auto-Correlation} for long-term series forecasting.
\newblock In {\em NeurIPS}, 2021.

\bibitem{cora}
Wenhan Xia, Chengwei Qin, and Elad Hazan.
\newblock Chain of lora: Efficient fine-tuning of language models via residual learning.
\newblock {\em arXiv preprint arXiv:2401.04151}, 2024.

\bibitem{CV-FT}
Yi~Xin, Siqi Luo, Haodi Zhou, Junlong Du, Xiaohong Liu, Yue Fan, Qing Li, and Yuntao Du.
\newblock Parameter-efficient fine-tuning for pre-trained vision models: A survey.
\newblock {\em arXiv preprint arXiv:2402.02242}, 2024.

\bibitem{FBM}
Runze Yang, Longbing Cao, JIE YANG, et~al.
\newblock Rethinking fourier transform from a basis functions perspective for long-term time series forecasting.
\newblock {\em Advances in Neural Information Processing Systems}, 37:8515--8540, 2024.

\bibitem{loretta}
Yifan Yang, Jiajun Zhou, Ngai Wong, and Zheng Zhang.
\newblock Loretta: Low-rank economic tensor-train adaptation for ultra-low-parameter fine-tuning of large language models.
\newblock {\em arXiv preprint arXiv:2402.11417}, 2024.

\bibitem{FourierGNN}
Kun Yi, Qi~Zhang, Wei Fan, Hui He, Liang Hu, Pengyang Wang, Ning An, Longbing Cao, and Zhendong Niu.
\newblock Fouriergnn: Rethinking multivariate time series forecasting from a pure graph perspective.
\newblock In {\em NeurIPS}, 2023.

\bibitem{FreTS}
Kun Yi, Qi~Zhang, Wei Fan, Shoujin Wang, Pengyang Wang, Hui He, Ning An, Defu Lian, Longbing Cao, and Zhendong Niu.
\newblock Frequency-domain mlps are more effective learners in time series forecasting.
\newblock In {\em NeurIPS}, 2023.

\bibitem{DLinear}
Ailing Zeng, Muxi Chen, Lei Zhang, and Qiang Xu.
\newblock Are transformers effective for time series forecasting?
\newblock In {\em AAAI}, 2023.

\bibitem{weather_app}
Gang Zhang, Dazhi Yang, George Galanis, and Emmanouil Androulakis.
\newblock Solar forecasting with hourly updated numerical weather prediction.
\newblock {\em Renewable and Sustainable Energy Reviews}, 154:111768, 2022.

\bibitem{adalora}
Qingru Zhang, Minshuo Chen, Alexander Bukharin, Nikos Karampatziakis, Pengcheng He, Yu~Cheng, Weizhu Chen, and Tuo Zhao.
\newblock Adalora: Adaptive budget allocation for parameter-efficient fine-tuning.
\newblock {\em arXiv preprint arXiv:2303.10512}, 2023.

\bibitem{film}
Tian Zhou, Ziqing Ma, Qingsong Wen, Liang Sun, Tao Yao, Wotao Yin, Rong Jin, et~al.
\newblock Film: Frequency improved legendre memory model for long-term time series forecasting.
\newblock {\em Advances in neural information processing systems}, 35:12677--12690, 2022.

\bibitem{fedformer}
Tian Zhou, Ziqing Ma, Qingsong Wen, Xue Wang, Liang Sun, and Rong Jin.
\newblock {FEDformer}: Frequency enhanced decomposed transformer for long-term series forecasting.
\newblock In {\em ICML}, 2022.

\bibitem{TSF-FT3}
Tian Zhou, Peisong Niu, Liang Sun, Rong Jin, et~al.
\newblock One fits all: Power general time series analysis by pretrained lm.
\newblock {\em Advances in neural information processing systems}, 36:43322--43355, 2023.

\end{thebibliography}
}


\newpage
\appendix

\section{Related Work}
\subsection{Time Series Forecasting Modeling}
Time-series forecasting (TSF) modeling generally involves encoding historical sequences and decoding temporal representations. To exploit temporal dynamics in the encoding phase, various deep learning backbones have been developed, generally grouped into four main categories: RNN-based (\eg SegRNN~\citep{segrnn}), CNN-based (\eg TimesNet~\citep{Timesnet}), GNN-based (\eg MAGNN~\citep{magnn}), MLP-based, and Transformer-based methods. 
Recent debates focus on MLP versus Transformer architectures, where MLPs (\eg DLinear~\citep{DLinear}, TSMixer~\citep{TSMixer2}) are efficient but limited in handling complex temporal patterns, whereas Transformers (\eg PatchTST~\citep{PatchTST}, iTransformer~\citep{itransformer}) excel in temporal encoding but are computationally intensive.
To better capture intricate temporal patterns, specialized designs such as series decomposition (\eg Autoformer~\citep{Autoformer}) and multiperiodicity analysis (\eg FiLM~\citep{film}) have been proposed, addressing seasonal and mixed period forecasting, respectively.
Recent innovations explore frequency-domain representations for temporal patterns, exemplified by FedFormer~\citep{fedformer} which employs frequency-domain attention score computation to enhance both efficiency and effectiveness. This paradigm demonstrates remarkable adaptability across architectures, including Transformers~\citep{fedformer,Autoformer}, MLPs~\citep{FreTS}, and GNNs~\citep{FourierGNN,stemgnn}, establishing frequency-domain analysis as a versatile plug-and-play component for temporal modeling.

Despite significant advancements in the encoding phase, the decoding phase remains underexplored in TSF modeling. The dominant multi-task forecasting (MT-F) paradigm employs simplistic affine transformations to map temporal representations to multi-step predictions. While computationally efficient, this approach inherently suffers from an \texttt{expressiveness bottleneck}, where predictions are confined to linear combinations of shared representation bases. Our theoretical analysis reveals that such constrained decoding mechanisms induce systematic error accumulation and suboptimal prediction expressiveness, even with well-encoded temporal patterns, exposing fundamental limitations in conventional MT-F paradigm.

\subsection{Modularized Fine-tuning}
Fine-tuning has emerged as a pivotal technique for adapting pre-trained models to downstream tasks, initially gaining prominence in natural language processing and computer vision~\citep{CV-FT}. Modern approaches have evolved into modular frameworks categorized into three paradigms: Adapter-based, Selection-based, and LoRA-based methods.
The Adapter-based paradigm introduces lightweight task-specific modules between pre-trained layers, preserving original parameters while enabling domain adaptation~\citep{serial-adapter,parallel_adapter,CoDA,mera}. In contrast, Selection-based methods employ parameter masking strategies to identify critical subnetworks for task-specific tuning~\citep{PaFi,SPT,FAR,cross_attn}.
The LoRA-based paradigm marks a significant technical evolution in parameter-efficient fine-tuning by introducing Low-Rank Adaptation (LoRA)~\citep{LoRA,note_on_lora}, which augments certain layers of a pre-trained model with trainable low-rank matrices. Instead of updating the full set of model parameters, LoRA only optimizes a small, low-rank decomposition of the weight update, thereby drastically reducing memory and computational overhead. This design allows for efficient adaptation to new tasks with minimal storage and enables fast switching between tasks by maintaining separate LoRA weights. LoRA’s effectiveness and versatility have been further demonstrated through updated matrix decomposition~\citep{loretta,adalora}, quantization~\citep{qlora,loftq}, and ranking adaptation~\citep{sora,cora}, which reflects a broader trend toward modularized and scalable LoRA applications.

In the context of TSF, LoRA-based fine-tuning has demonstrated promising capabilities from two main perspectives. On one hand, \citep{TSF-FT5} introduces channel-aware LoRA, leveraging low-rank adaptation to capture channel dependencies. On the other hand, a series of studies~\citep{TSF-FT1,TSF-FT2,TSF-FT3,TSF-FT4} investigate the impact of LoRA within time series foundation models, focusing on how LoRA facilitates efficient adaptation and task transfer in time series forecasting.
Though these explorations are promising, they largely overlook two critical aspects: the expressiveness bottleneck we have identified in TSF, and the untapped potential of LoRA for overcoming this limitation. Our work diverges by specifically addressing the MT-F paradigm’s expressiveness bottleneck through innovative integration of temporal segmentation and Mixture-of-LoRA enhanced decoding, ensuring both model efficiency and expressiveness.

\section{Complexity Analysis}
\label{apdx:efficiency}
In this section, we conduct a parameter count analysis and evaluate the running cost of MoLA through empirical investigation. 
Take iTransformer~\citep{itransformer} as the base model, let the number of layers in iTransformer be $\N_l$, the hidden dimension of the attention layer be $d_{\text{m}}$, and the hidden dimension of the FFN layer be $d_{\text{ff}}$.
Then, the number of parameters introduced by MoLA is given by:
\begin{equation}
\begin{aligned}
    \N_{\text{MoLA}} &= \N_l \times 2 \times (\N_{\text{LoRA}}\times \mP + \N_{\text{Weight}} \times \K) \\
    & = \N_l \times 2 \times ((d_{\text{m}} \times r + r \times d_{\text{ff}}) \times \mP + \mP \times \K),
\end{aligned}
\end{equation}
where $\N_{\text{LoRA}}$ and $\N_{\text{Weight}}$ are the number of parameters introduced by LoRA expert module and learnable weights for each expert, respectively.

Consider that iTransformer is a standard Transformer architecture with dimension permutation in the input sequence, its number of parameters can be given by:
\begin{equation}
\begin{aligned}
    \N_{\text{iTrans}} &= \N_l \times (\N_{\text{Att}} + \N_{\text{FFN}} + \N_{\text{LN}}) \\
    &= \N_l \times (4\times d_{\text{m}}^2 + 2\times d_{\text{m}} \times d_{\text{ff}} + 4\times d_{\text{m}}),
\end{aligned}
\end{equation}
where $\N_{\text{Att}}$, $\N_{\text{FFN}}$ and $\N_{\text{LN}}$ are the number of parameters introduced by self attention layer, feedforward layer and layer normalization layer, respectively.

The ratio between the two parameter counts is therefore:  
\begin{equation}
    \frac{\N_{\text{MoLA}}}{\N_{\text{iTrans}}} = \frac{\mP \times (d_{\text{m}}\times r + r \times d_{\text{ff}} + \K)}{2\times d_{\text{m}}\times (d_{\text{m}}+d_{\text{ff}} + 1)}.
\end{equation}

For example, with $\mP = 4$, $\K = 6$, $d_{\text{m}} = 512$, $d_{\text{ff}} = 1024$, and $r = 8$, the ratio $\frac{\N_{\text{MoLA}}}{\N_{\text{iTrans}}}$ is approximately 0.047, indicating a relevant small increase in model size.
Figure~\ref{fig:timing} shows MoLA's empirical running costs with varying $r$, $\mP$, and $\K$, capturing the sum time for weight matrix adaptation and base layer forward pass.
Our results confirm that the additional computational duration imposed by MoLA is lower than 1ms, with a small increase as the hyperparameters grow. Therefore, MoLA does not compromise the model's efficiency while effectively improving the model's performance.

\begin{figure}[h]
\begin{center}
\subfigure[Varying $r$ on ETTh1.]{\includegraphics[width=0.28\linewidth]{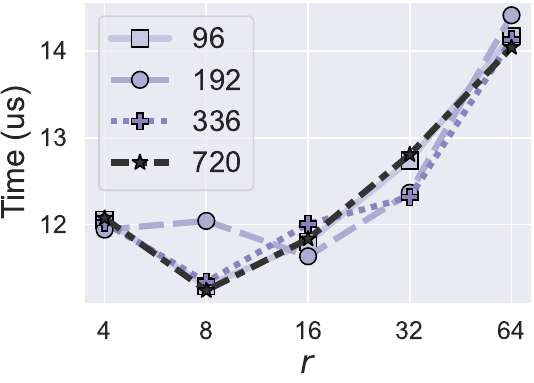}\label{subfig:timing-a}}
\hspace{8pt}
\subfigure[Varying $\mP$ on ETTh1.]{\includegraphics[width=0.28\linewidth]{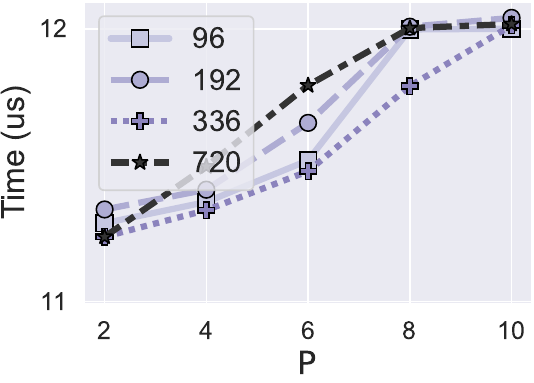}\label{subfig:timing-c}}
\hspace{8pt}
\subfigure[Varying $\K$ on ETTh1.]{\includegraphics[width=0.28\linewidth]{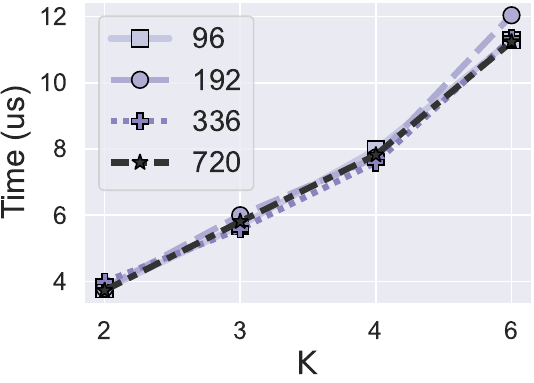}\label{subfig:timing-e}}

\subfigure[Varying $r$ on Weather.]{\includegraphics[width=0.28\linewidth]{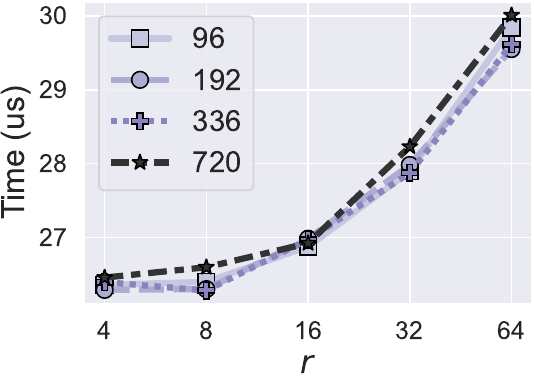}\label{subfig:timing-b}}
\hspace{8pt}
\subfigure[Varying $\mP$ on Weather.]{\includegraphics[width=0.28\linewidth]{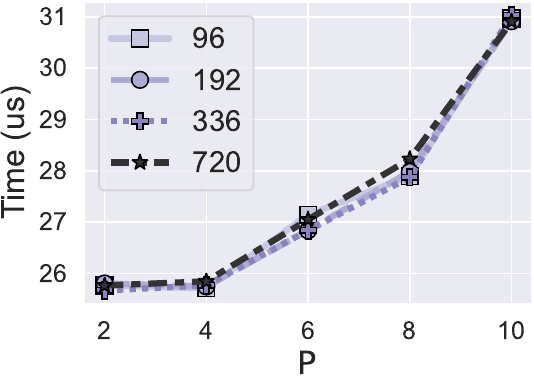}\label{subfig:timing-d}}
\hspace{8pt}
\subfigure[Varying $\K$ on Weather.]{\includegraphics[width=0.28\linewidth]{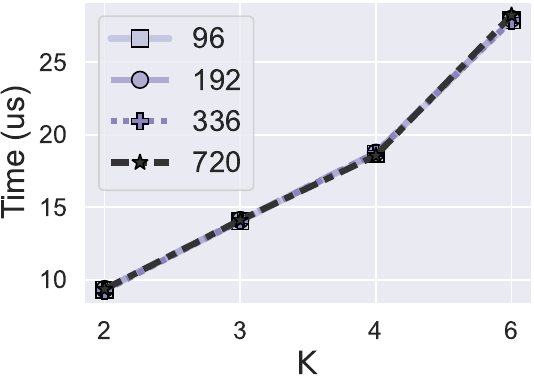}\label{subfig:timing-f}}
\caption{Running time of MoLA given varying rank $r$, the number of experts $\mP$ and the number of segments $\K$.}
\label{fig:timing}
\end{center}
\vskip -0.2in
\end{figure}

\section{Theoretical Justification}



\begin{theorem}[Expressiveness Bottleneck]
\label{apdx_thm:bottleneck_quantification}
Let $\bar{W} = [W \ b] \in \mathbb{R}^{\T \times (\mL+1)}$ be the parameters in the MT-F's linear decoder, $Y\in\mathbb{R}^{\T\times\mathrm{D}}$ be the label sequence;  the minimum attainable estimation error is
\begin{equation}
    \|\epsilon\|_F^2 = \sum_{t=\mathrm{rank}(\bar{W})+1}^{\T} \|U_t^\top Y\|_2^2,
\end{equation}
where $\bar{W} = U\Sigma V^\top$ is the singular value decomposition of $\bar{W}$,
$\mathrm{rank}(\bar{W}) \leq \min\{\T, \mL+1\}$, and $\{U_i\}_{i=\mathrm{rank}+1}^{\T}$ form an orthonormal basis for the null space of $\bar{W}$. Notably, this error is independent of the representation $R$ provided by encoder.
\end{theorem}

\begin{proof}
Consider the least squares estimation, the aim of linear decoder is to find $\bar{R}=[R^\top, 1]^\top \in \mathbb{R}^{(\mL + 1) \times 1}$ such that:
\begin{align*}
    \hat{\bar{R}} = \arg\min_{\bar{R}} \|Y - \bar{W} \bar{R}\|_2^2.
\end{align*}

To solve this optimization problem, we set the gradient of the cost function with respect to $\bar{R}$ to zero:
\begin{align*}
    \frac{\partial}{\partial \bar{R}} \|Y - \bar{W} \bar{R}\|_2^2 = -2 \bar{W}^\top (Y - \bar{W} \bar{R}) = 0.
\end{align*}

Assuming $\bar{W}^\top \bar{W}$ is invertible, we have:
\begin{align*}
\bar{R} = (\bar{W}^\top \bar{W})^{-1} \bar{W}^\top Y.
\end{align*}

Then, the estimation error is:
\begin{align*}
\epsilon = Y - \hat{Y} = \left( I - \bar{W} (\bar{W}^\top \bar{W})^{-1} \bar{W}^\top \right) Y = (I-P)Y,
\end{align*}
where $P = \bar{W} (\bar{W}^\top \bar{W})^{-1} \bar{W}^\top$ is known as the projection matrix onto the column space of $\bar{W}$. Therefore, $I - P$ projects onto the orthogonal complement of the column space of $\bar{W}$.

Let the Singular Value Decomposition (SVD) of \( \bar{W} \) be
\begin{align*}
    \bar{W} = U \Sigma V^\top,
\end{align*}
where \( U \in \mathbb{R}^{\T \times \T} \) is orthogonal, \( \Sigma \in \mathbb{R}^{\T \times (\mL+1)} \) is diagonal, and \( V \in \mathbb{R}^{(\mL+1) \times (\mL+1)} \) is orthogonal.

Partition \( \Sigma \) and \( U \) as
\[
\Sigma = \begin{bmatrix} \Sigma_o & 0 \\ 0 & 0 \end{bmatrix}, \quad U = [U_o \quad U_n],
\]
where \( \Sigma_o \in \mathbb{R}^{o \times o} \) contains the positive singular values, \( U_o \in \mathbb{R}^{\T \times o} \) contains the corresponding left singular vectors, and \( U_n \in \mathbb{R}^{\T \times (\T - o)} \) contains the left singular vectors corresponding to zero singular values, and $o=\mathrm{rank}(\bar{W})$.

The projection matrix onto the column space of \( \bar{W} \) is
\begin{align*}
    P = \bar{W} (\bar{W}^\top \bar{W})^{-1} \bar{W}^\top = U_o U_o^\top.
\end{align*}
Therefore, the estimation error is
\begin{align*}
    \epsilon = (I - P) Y = U_n U_n^\top Y.
\end{align*}
The minimum attainable estimation error is
\begin{align*}
    \|\epsilon\|_F^2 = \| U_n^\top Y \|_2^2 = \sum_{t = \mathrm{rank}(\bar{W})+1}^{\T} \left\| U_t^\top Y \right\|_2^2.
\end{align*}

\end{proof}

\begin{theorem}[Variance Reduction of MoLA]
\label{apdx_thm:variance_reduction}
$\mathcal{L}_{\text{MoLA}}$ has a smaller variance than $\mathcal{L}_{\text{MT-F}}$,
\begin{equation}
    \operatorname{Var}(\mathcal{L}_{\text{MoLA}}) \leq \operatorname{Var}(\mathcal{L}_{\text{MT-F}}).
\end{equation}
\end{theorem}

\begin{proof}
Consider the total loss expressed as the average of losses over the prediction horizons:
\begin{align*}
    \mathcal{L} = \frac{1}{\T} \sum_{t=1}^\T \mathcal{L}_t,
\end{align*}
where $\mathcal{L}_t$ is the MSE loss for the $t$-th prediction horizon.

The variance of the total loss is:
\begin{align*}
\begin{aligned}
    \operatorname{Var}(\mathcal{L}) &= \operatorname{Var}\left( \frac{1}{\T} \sum_{t=1}^\T \mathcal{L}_t \right) \\
    &= \frac{1}{\T^2} \left( \sum_{t=1}^\T \operatorname{Var}(\mathcal{L}_t) + 2 \sum_{1 \leq t < s \leq \T} \operatorname{Cov}(\mathcal{L}_t, \mathcal{L}_s) \right).
\end{aligned}
\end{align*}

Under modularized fine-tuning, the covariances between different $\mathcal{L}_t$ decrease because the LoRA modules introduce horizon-specific parameters, reducing parameter sharing. Let $\Delta\operatorname{Cov}(\mathcal{L}_t, \mathcal{L}_s) = \operatorname{Cov}_{\text{MT-F}}(\mathcal{L}_t, \mathcal{L}_s) - \operatorname{Cov}_{\text{MoLA}}(\mathcal{L}_t, \mathcal{L}_s) \geq 0$. The variance difference is then:
\begin{align*}
    \operatorname{Var}(\mathcal{L}_{\text{MT-F}}) - \operatorname{Var}(\mathcal{L}_{\text{MoLA}}) = \frac{2}{\T^2} \sum_{1 \leq t < s \leq \T} \Delta\operatorname{Cov}(\mathcal{L}_t, \mathcal{L}_s) \geq 0.
\end{align*}

Therefore,
\begin{align*}
\begin{aligned}
    \operatorname{Var}(\mathcal{L}_{\text{MoLA}}) &= \operatorname{Var}(\mathcal{L}_{\text{MT-F}}) - \frac{2}{\T^2} \sum_{1 \leq t < s \leq \T} \Delta\operatorname{Cov}(\mathcal{L}_t, \mathcal{L}_s) \\
    &\leq \operatorname{Var}(\mathcal{L}_{\text{MT-F}}).
\end{aligned}
\end{align*}

The equation realize only when all the LoRA modules do not contribute to the improvement of the prediction performance with corresponding time step.
However, we didn't recognize that phenomenon in our experiments.

\end{proof}

\section{Reproduction Details}


\subsection{Dataset Descriptions}
\label{apdx:dataset}

\begin{table}
  \caption{Detailed dataset descriptions. \textit{D} denotes the number of variates. \texttt{Forecast Length} denotes the prediction lengths investigated in this dataset. \texttt{Frequency} denotes the sampling interval of time points. \texttt{Train, Validation, Test} denotes the number of samples employed in each split. The taxonomy and statistic are aligned with the recent works~\citep{Timesnet,itransformer}.}\label{tab:dataset}
  \centering
  \renewcommand{\multirowsetup}{\centering}
  \setlength{\tabcolsep}{10pt}
  \footnotesize
  \begin{threeparttable}
  \begin{tabular}{llllll}
    \toprule
    Dataset & D & Forecast Length & Split Ratio & Frequency& Domain \\
    \toprule
    ETTm1 & 7 & 96,192,336,720 & 34465/11521/11521 & 15mins & Electricity\\
    \midrule
    ETTm2 & 7 & 96,192,336,720 & 34465/11521/11521 & 15mins & Electricity\\
    \midrule
    ETTh1 & 7 & 96,192,336,720 & 8545/2881/2881 & Hourly & Electricity\\
    \midrule
    ETTh2 & 7 & 96,192,336,720 & 8545/2881/2881 & Hourly & Electricity\\
    \midrule
    ECL & 321 & 96,192,336,720 & 18317/2633/5261 & Hourly & Electricity \\
    \midrule
    Traffic & 862 & 96,192,336,720 & 12185/1757/3509 & Hourly & Transportation \\
    \midrule
    Weather & 21 & 96,192,336,720 & 36792/5271/10540 & 10mins & Weather\\
    \midrule
    PEMS03 & 358 & 12,24,36,48 & 15617/5135/5135 & 5mins & Transportation\\
    \midrule
    PEMS08 & 170 & 12,24,36,48 & 10690/3548/265 & 5mins & Transportation\\
    \bottomrule
  \end{tabular}
  \end{threeparttable}
\end{table}

The datasets used in this study span a variety of domains and time resolutions, each with distinct characteristics that are well-suited for evaluating time-series forecasting:
\begin{itemize}[leftmargin=*]
\item \textbf{ETT}~\citep{Informer}: This dataset consists of data from 7 key variables related to electricity transformers, collected between July 2016 and July 2018. It includes four subsets: ETTh1 and ETTh2, which are recorded hourly, and ETTm1 and ETTm2, which are recorded every 15 minutes.
\item \textbf{ECL (Electricity Consumption Load)}~\citep{Autoformer}: This dataset contains hourly electricity consumption data from 321 clients, offering insights into power usage patterns over time.
\item \textbf{Traffic}~\citep{Autoformer}: This dataset comprises hourly road occupancy rates collected by 862 sensors deployed across the freeways in the San Francisco Bay Area. The data spans the period from January 2015 to December 2016, reflecting traffic conditions over time.
\item \textbf{Weather}~\citep{Autoformer}: This dataset includes 21 meteorological variables, recorded every 10 minutes throughout 2020 at the Max Planck Biogeochemistry Institute's weather station. It provides a comprehensive set of climate-related factors for forecasting purposes.
\item \textbf{PEMS}~\citep{SCINet}: This dataset consists of public traffic data from the California highway network, with recordings collected every 5 minutes. We utilize two subsets in this study: PEMS03 and PEMS08, which are frequently adopted in traffic forecasting benchmarks.
\end{itemize}

For all datasets, the data processing and division into training, validation, and test sets follow the protocols established by TimesNet~\citep{Timesnet} and iTransformer~\citep{itransformer}, ensuring a consistent chronological split to avoid data leakage. The standardized lookback window is set at $96$ for the ETT, ECL, Traffic, Weather, and PEMS datasets. Prediction horizons vary across datasets, with forecasting lengths of $96$, $192$, $336$, and $720$ for the first five datasets, and shorter horizons of $12$, $24$, $36$, and $48$ for the PEMS subsets. Detailed specifications of each dataset can be found in Table~\ref{tab:dataset}.

\subsection{Implementation Details}
\label{apdx:imple}
The baseline models in this study were carefully reproduced using training scripts from the TimesNet Repository~\citep{Timesnet}, ensuring full reproducibility verification. All models were trained using the Adam optimizer~\citep{Adam}, with learning rates selected from $\{10^{-3}, 5\times 10^{-4}, 10^{-4}\}$ based on minimizing the MSE loss. A batch size of 32 was maintained consistently across all experiments. Training was performed for a maximum of 10 epochs, with an early stopping criterion triggered if no improvement in validation performance was observed for 3 consecutive epochs.

For the experiments integrating MoLA into existing forecasting models, we strictly adhered to the original hyperparameter settings as outlined in the respective publications during the pre-training phase. The forecasting horizon of pre-trained models were set to $\T / \K$, where $\K \in \{2, 3, 4, 6\}$. Pre-training was limited to a maximum of 5 epochs, with early stopping applied after 2 epochs without improvement.
During the fine-tuning phase, only the rank $r$ in the LoRA expert module, the learning rate $\eta$ and the number of LoRA expert modules $\mP$ were tuned, as these parameters are crucial for adjusting the differences in weight magnitudes between the base models and the LoRA expert modules. Fine-tuning was performed to minimize the MSE averaged over all prediction lengths, with forecasting segments selected from range $\{2, 3, 4, 6\}$. The current results sufficiently demonstrate the effectiveness of MoLA, showing that its efficacy is not dependent on highly specific hyperparameter configurations.

\section{More Experimental Results}
\subsection{Overall Performance}
\begin{table}
  \caption{Full results on the multi-step forecasting task. The length of history window is set to 96 for all baselines. \texttt{Avg} indicates the results averaged over forecasting lengths: T=96, 192, 336 and 720.}\label{tab:multistep_app_full}
  \renewcommand{\arraystretch}{0.8}
  \setlength{\tabcolsep}{1.6pt}
  \tiny
  \centering
  \renewcommand{\multirowsetup}{\centering}
  \begin{tabular}{c|c|cc|cc|cc|cc|cc|cc|cc|cc|cc|cc|cc}
    \toprule
    \multicolumn{2}{l}{\multirow{2}{*}{\rotatebox{0}{\scaleb{Models}}}} & 
    \multicolumn{2}{c}{\rotatebox{0}{\scaleb{\textbf{MoLA}}}} &
    \multicolumn{2}{c}{\rotatebox{0}{\scaleb{iTransformer}}} &
    \multicolumn{2}{c}{\rotatebox{0}{\scaleb{FreTS}}} &
    \multicolumn{2}{c}{\rotatebox{0}{\scaleb{TimesNet}}} &
    \multicolumn{2}{c}{\rotatebox{0}{\scaleb{TiDE}}} &
    \multicolumn{2}{c}{\rotatebox{0}{\scaleb{DLinear}}} &
    \multicolumn{2}{c}{\rotatebox{0}{\scaleb{FEDformer}}} &
    \multicolumn{2}{c}{\rotatebox{0}{\scaleb{Autoformer}}} &
    \multicolumn{2}{c}{\rotatebox{0}{\scaleb{Informer}}} &
    \multicolumn{2}{c}{\rotatebox{0}{\scaleb{Transformer}}} &
    \multicolumn{2}{c}{\rotatebox{0}{\scaleb{TCN}}} \\
    \multicolumn{2}{c}{} &
    \multicolumn{2}{c}{\scaleb{\textbf{(Ours)}}} & 
    \multicolumn{2}{c}{\scaleb{(2024)}} & 
    \multicolumn{2}{c}{\scaleb{(2023)}} & 
    \multicolumn{2}{c}{\scaleb{(2023)}} & 
    \multicolumn{2}{c}{\scaleb{(2023)}} & 
    \multicolumn{2}{c}{\scaleb{(2023)}} & 
    \multicolumn{2}{c}{\scaleb{(2022)}} &
    \multicolumn{2}{c}{\scaleb{(2021)}} &
    \multicolumn{2}{c}{\scaleb{(2021)}} &
    \multicolumn{2}{c}{\scaleb{(2017)}} &
    \multicolumn{2}{c}{\scaleb{(2017)}} \\
    \cmidrule(lr){3-4} \cmidrule(lr){5-6}\cmidrule(lr){7-8} \cmidrule(lr){9-10}\cmidrule(lr){11-12} \cmidrule(lr){13-14} \cmidrule(lr){15-16} \cmidrule(lr){17-18} \cmidrule(lr){19-20} \cmidrule(lr){21-22} \cmidrule(lr){23-24}
    \multicolumn{2}{l}{\rotatebox{0}{\scaleb{Metrics}}}  & \scalea{MSE} & \scalea{MAE}  & \scalea{MSE} & \scalea{MAE}  & \scalea{MSE} & \scalea{MAE}  & \scalea{MSE} & \scalea{MAE}  & \scalea{MSE} & \scalea{MAE}  & \scalea{MSE} & \scalea{MAE} & \scalea{MSE} & \scalea{MAE} & \scalea{MSE} & \scalea{MAE} & \scalea{MSE} & \scalea{MAE} & \scalea{MSE} & \scalea{MAE} & \scalea{MSE} & \scalea{MAE} \\
    \toprule
    
    \multirow{5}{*}{{\rotatebox{90}{\scalebox{0.95}{ETTm1}}}}
    & \scalea{96} 
    & \scalea{\bst{0.330}} & \scalea{\bst{0.364}} 
    & \scalea{0.346} & \scalea{0.379} 
    & \scalea{0.339} & \scalea{0.374} 
    & \scalea{\subbst{0.338}} & \scalea{0.379} 
    & \scalea{0.364} & \scalea{0.387} 
    & \scalea{0.345} & \scalea{\subbst{0.372}} 
    & \scalea{0.389} & \scalea{0.427} 
    & \scalea{0.468} & \scalea{0.463} 
    & \scalea{0.633} & \scalea{0.560}
    & \scalea{0.591} & \scalea{0.549} 
    & \scalea{0.887} & \scalea{0.613} \\
    
    & \scalea{192} 
    & \scalea{\bst{0.378}} & \scalea{\subbst{0.392}} 
    & \scalea{0.392} & \scalea{0.400} 
    & \scalea{0.382} & \scalea{0.397} 
    & \scalea{0.389} & \scalea{0.400} 
    & \scalea{0.398} & \scalea{0.404} 
    & \scalea{\subbst{0.381}} & \scalea{\bst{0.390}} 
    & \scalea{0.402} & \scalea{0.431} 
    & \scalea{0.573} & \scalea{0.509} 
    & \scalea{0.736} & \scalea{0.625}
    & \scalea{0.704} & \scalea{0.629} 
    & \scalea{0.877} & \scalea{0.626} \\
    
    & \scalea{336} 
    & \scalea{\bst{0.414}} & \scalea{\subbst{0.416}} 
    & \scalea{0.427} & \scalea{0.422} 
    & \scalea{0.421} & \scalea{0.426} 
    & \scalea{0.429} & \scalea{0.428} 
    & \scalea{0.428} & \scalea{0.425} 
    & \scalea{\subbst{0.414}} & \scalea{\bst{0.414}} 
    & \scalea{0.438} & \scalea{0.451} 
    & \scalea{0.596} & \scalea{0.527} 
    & \scalea{1.061} & \scalea{0.787}
    & \scalea{1.171} & \scalea{0.861} 
    & \scalea{0.890} & \scalea{0.636} \\
    
    & \scalea{720} 
    & \scalea{\subbst{0.480}} & \scalea{\subbst{0.453}} 
    & \scalea{0.494} & \scalea{0.461} 
    & \scalea{0.485} & \scalea{0.462} 
    & \scalea{0.495} & \scalea{0.464} 
    & \scalea{0.487} & \scalea{0.461} 
    & \scalea{\bst{0.473}} & \scalea{\bst{0.451}} 
    & \scalea{0.529} & \scalea{0.498} 
    & \scalea{0.749} & \scalea{0.569} 
    & \scalea{1.119} & \scalea{0.801}
    & \scalea{1.307} & \scalea{0.893} 
    & \scalea{0.911} & \scalea{0.653} \\
    \cmidrule(lr){2-24}
    
    & \scalea{Avg} 
    & \scalea{\bst{0.400}} & \scalea{\bst{0.406}} 
    & \scalea{0.415} & \scalea{0.416} 
    & \scalea{0.407} & \scalea{0.415} 
    & \scalea{0.413} & \scalea{0.418} 
    & \scalea{0.419} & \scalea{0.419} 
    & \scalea{\subbst{0.404}} & \scalea{\subbst{0.407}} 
    & \scalea{0.440} & \scalea{0.451} 
    & \scalea{0.596} & \scalea{0.517}
    & \scalea{0.887} & \scalea{0.693}
    & \scalea{0.943} & \scalea{0.733} 
    & \scalea{0.891} & \scalea{0.632} \\
    \midrule

    \multirow{5}{*}{{\rotatebox{90}{\scalebox{0.95}{ETTm2}}}}
    & \scalea{96}  
    & \scalea{\bst{0.181}} & \scalea{\bst{0.262}} 
    & \scalea{\subbst{0.184}} & \scalea{0.266} 
    & \scalea{0.190} & \scalea{0.282} 
    & \scalea{0.185} & \scalea{\subbst{0.264}} 
    & \scalea{0.207} & \scalea{0.305} 
    & \scalea{0.195} & \scalea{0.294} 
    & \scalea{0.194} & \scalea{0.284} 
    & \scalea{0.240} & \scalea{0.319} 
    & \scalea{0.541} & \scalea{0.581}
    & \scalea{0.317} & \scalea{0.408} 
    & \scalea{3.125} & \scalea{1.345} \\
    
    & \scalea{192} 
    & \scalea{\bst{0.247}} & \scalea{\bst{0.307}} 
    & \scalea{0.257} & \scalea{0.315} 
    & \scalea{0.260} & \scalea{0.329} 
    & \scalea{\subbst{0.254}} & \scalea{\subbst{0.307}} 
    & \scalea{0.290} & \scalea{0.364} 
    & \scalea{0.283} & \scalea{0.359} 
    & \scalea{0.264} & \scalea{0.324} 
    & \scalea{0.300} & \scalea{0.349} 
    & \scalea{0.527} & \scalea{0.558}
    & \scalea{1.069} & \scalea{0.758} 
    & \scalea{3.130} & \scalea{1.350} \\
    
    & \scalea{336} 
    & \scalea{\bst{0.312}} & \scalea{\subbst{0.347}}
    & \scalea{0.315} & \scalea{0.351} 
    & \scalea{0.373} & \scalea{0.405} 
    & \scalea{\subbst{0.314}} & \scalea{\bst{0.345}} 
    & \scalea{0.377} & \scalea{0.422} 
    & \scalea{0.384} & \scalea{0.427} 
    & \scalea{0.319} & \scalea{0.359} 
    & \scalea{0.339} & \scalea{0.375} 
    & \scalea{1.126} & \scalea{0.797}
    & \scalea{1.325} & \scalea{0.869} 
    & \scalea{3.185} & \scalea{1.375} \\
    
    & \scalea{720} 
    & \scalea{\bst{0.407}} & \scalea{\bst{0.403}} 
    & \scalea{\subbst{0.419}} & \scalea{\subbst{0.409}} 
    & \scalea{0.517} & \scalea{0.499} 
    & \scalea{0.434} & \scalea{0.413} 
    & \scalea{0.558} & \scalea{0.524} 
    & \scalea{0.516} & \scalea{0.502} 
    & \scalea{0.430} & \scalea{0.424} 
    & \scalea{0.423} & \scalea{0.421} 
    & \scalea{2.828} & \scalea{1.268}
    & \scalea{2.576} & \scalea{1.223} 
    & \scalea{4.203} & \scalea{1.658} \\
    \cmidrule(lr){2-24}
    
    & \scalea{Avg} 
    & \scalea{\bst{0.287}} & \scalea{\bst{0.330}} 
    & \scalea{\subbst{0.294}} & \scalea{0.335} 
    & \scalea{0.335} & \scalea{0.379} 
    & \scalea{0.297} & \scalea{\subbst{0.332}} 
    & \scalea{0.358} & \scalea{0.404}
    & \scalea{0.344} & \scalea{0.396} 
    & \scalea{0.302} & \scalea{0.348} 
    & \scalea{0.326} & \scalea{0.366} 
    & \scalea{1.256} & \scalea{0.801}
    & \scalea{1.322} & \scalea{0.814} 
    & \scalea{3.411} & \scalea{1.432} \\
    \midrule

    \multirow{5}{*}{\rotatebox{90}{{\scalebox{0.95}{ETTh1}}}}
    & \scalea{96} 
    & \scalea{\subbst{0.379}} & \scalea{\bst{0.400}} 
    & \scalea{0.390} & \scalea{\subbst{0.410}} 
    & \scalea{0.399} & \scalea{0.412} 
    & \scalea{0.422} & \scalea{0.433} 
    & \scalea{0.479} & \scalea{0.464} 
    & \scalea{0.396} & \scalea{0.410} 
    & \scalea{\bst{0.377}} & \scalea{0.418} 
    & \scalea{0.423} & \scalea{0.441} 
    & \scalea{0.920} & \scalea{0.745}
    & \scalea{0.796} & \scalea{0.691} 
    & \scalea{0.767} & \scalea{0.633} \\
    
    & \scalea{192} 
    & \scalea{\subbst{0.436}} & \scalea{\bst{0.432}} 
    & \scalea{0.443} & \scalea{\subbst{0.441}} 
    & \scalea{0.453} & \scalea{0.443} 
    & \scalea{0.465} & \scalea{0.457} 
    & \scalea{0.521} & \scalea{0.503} 
    & \scalea{0.449} & \scalea{0.444} 
    & \scalea{\bst{0.421}} & \scalea{0.445} 
    & \scalea{0.498} & \scalea{0.485} 
    & \scalea{0.998} & \scalea{0.781}
    & \scalea{0.813} & \scalea{0.699} 
    & \scalea{0.739} & \scalea{0.619} \\
    
    & \scalea{336} 
    & \scalea{\subbst{0.477}} & \scalea{\bst{0.456}} 
    & \scalea{0.480} & \scalea{\subbst{0.457}} 
    & \scalea{0.503} & \scalea{0.475} 
    & \scalea{0.492} & \scalea{0.470} 
    & \scalea{0.659} & \scalea{0.603} 
    & \scalea{0.487} & \scalea{0.465} 
    & \scalea{\bst{0.468}} & \scalea{0.472} 
    & \scalea{0.506} & \scalea{0.496} 
    & \scalea{1.091} & \scalea{0.812}
    & \scalea{1.181} & \scalea{0.876} 
    & \scalea{0.717} & \scalea{0.613} \\
    
    & \scalea{720} 
    & \scalea{\subbst{0.480}} & \scalea{\bst{0.478}} 
    & \scalea{0.484} & \scalea{\subbst{0.479}} 
    & \scalea{0.596} & \scalea{0.565} 
    & \scalea{0.532} & \scalea{0.502} 
    & \scalea{0.893} & \scalea{0.736} 
    & \scalea{0.516} & \scalea{0.513} 
    & \scalea{0.500} & \scalea{0.493} 
    & \scalea{\bst{0.477}} & \scalea{0.487} 
    & \scalea{1.247} & \scalea{0.887}
    & \scalea{1.182} & \scalea{0.885} 
    & \scalea{0.828} & \scalea{0.678} \\
    \cmidrule(lr){2-24}
    
    & \scalea{Avg} 
    & \scalea{\subbst{0.443}} & \scalea{\bst{0.441}} 
    & \scalea{0.449} & \scalea{\subbst{0.447}} 
    & \scalea{0.488} & \scalea{0.474} 
    & \scalea{0.478} & \scalea{0.466} 
    & \scalea{0.628} & \scalea{0.574} 
    & \scalea{0.462} & \scalea{0.458} 
    & \scalea{\bst{0.441}} & \scalea{0.457} 
    & \scalea{0.476} & \scalea{0.477} 
    & \scalea{1.064} & \scalea{0.806}
    & \scalea{0.993} & \scalea{0.788} 
    & \scalea{0.763} & \scalea{0.636} \\
    \midrule

    \multirow{5}{*}{\rotatebox{90}{{\scalebox{0.95}{ETTh2}}}}
    & \scalea{96}  
    & \scalea{\bst{0.293}} & \scalea{\bst{0.341}} 
    & \scalea{\subbst{0.301}} & \scalea{\subbst{0.349}} 
    & \scalea{0.350} & \scalea{0.403} 
    & \scalea{0.320} & \scalea{0.364} 
    & \scalea{0.400} & \scalea{0.440} 
    & \scalea{0.343} & \scalea{0.396} 
    & \scalea{0.347} & \scalea{0.391} 
    & \scalea{0.383} & \scalea{0.424} 
    & \scalea{2.340} & \scalea{1.220}
    & \scalea{2.072} & \scalea{1.140} 
    & \scalea{3.171} & \scalea{1.364} \\
    
    & \scalea{192} 
    & \scalea{\bst{0.377}} & \scalea{\bst{0.397}} 
    & \scalea{\subbst{0.382}} & \scalea{\subbst{0.402}} 
    & \scalea{0.472} & \scalea{0.475} 
    & \scalea{0.409} & \scalea{0.417} 
    & \scalea{0.528} & \scalea{0.509} 
    & \scalea{0.473} & \scalea{0.474} 
    & \scalea{0.430} & \scalea{0.443} 
    & \scalea{0.557} & \scalea{0.511} 
    & \scalea{6.284} & \scalea{2.078}
    & \scalea{5.081} & \scalea{1.814} 
    & \scalea{3.222} & \scalea{1.398} \\
    
    & \scalea{336} 
    & \scalea{\bst{0.421}} & \scalea{\bst{0.429}} 
    & \scalea{\subbst{0.430}} & \scalea{\subbst{0.434}} 
    & \scalea{0.564} & \scalea{0.528} 
    & \scalea{0.449} & \scalea{0.451} 
    & \scalea{0.643} & \scalea{0.571} 
    & \scalea{0.603} & \scalea{0.546} 
    & \scalea{0.469} & \scalea{0.475} 
    & \scalea{0.470} & \scalea{0.481} 
    & \scalea{4.824} & \scalea{1.853}
    & \scalea{3.564} & \scalea{1.475} 
    & \scalea{3.306} & \scalea{1.452} \\
    
    & \scalea{720} 
    & \scalea{\bst{0.418}} & \scalea{\bst{0.438}} 
    & \scalea{\subbst{0.447}} & \scalea{\subbst{0.455}} 
    & \scalea{0.815} & \scalea{0.654} 
    & \scalea{0.473} & \scalea{0.474} 
    & \scalea{0.874} & \scalea{0.679} 
    & \scalea{0.812} & \scalea{0.650} 
    & \scalea{0.473} & \scalea{0.480} 
    & \scalea{0.501} & \scalea{0.515} 
    & \scalea{3.985} & \scalea{1.724}
    & \scalea{2.469} & \scalea{1.247} 
    & \scalea{3.599} & \scalea{1.565} \\
    \cmidrule(lr){2-24}
    
    & \scalea{Avg} 
    & \scalea{\bst{0.377}} & \scalea{\bst{0.401}} 
    & \scalea{\subbst{0.390}} & \scalea{\subbst{0.410}} 
    & \scalea{0.550} & \scalea{0.515} 
    & \scalea{0.413} & \scalea{0.426} 
    & \scalea{0.611} & \scalea{0.550} 
    & \scalea{0.558} & \scalea{0.516} 
    & \scalea{0.430} & \scalea{0.447} 
    & \scalea{0.478} & \scalea{0.483} 
    & \scalea{4.358} & \scalea{1.719}
    & \scalea{3.296} & \scalea{1.419} 
    & \scalea{3.325} & \scalea{1.445} \\
    \midrule

    \multirow{5}{*}{{\rotatebox{90}{\scalebox{0.95}{ECL}}}} 
    & \scalea{96} 
    & \scalea{\bst{0.147}} & \scalea{\bst{0.237}} 
    & \scalea{\subbst{0.148}} & \scalea{\subbst{0.239}} 
    & \scalea{0.189} & \scalea{0.277} 
    & \scalea{0.171} & \scalea{0.273} 
    & \scalea{0.237} & \scalea{0.329} 
    & \scalea{0.210} & \scalea{0.302} 
    & \scalea{0.200} & \scalea{0.315} 
    & \scalea{0.199} & \scalea{0.315} 
    & \scalea{0.315} & \scalea{0.398}
    & \scalea{0.252} & \scalea{0.352} 
    & \scalea{0.688} & \scalea{0.621} \\
    
    & \scalea{192} 
    & \scalea{\bst{0.163}} & \scalea{\bst{0.254}} 
    & \scalea{\subbst{0.167}} & \scalea{\subbst{0.258}} 
    & \scalea{0.193} & \scalea{0.282} 
    & \scalea{0.188} & \scalea{0.289} 
    & \scalea{0.236} & \scalea{0.330} 
    & \scalea{0.210} & \scalea{0.305} 
    & \scalea{0.207} & \scalea{0.322} 
    & \scalea{0.215} & \scalea{0.327} 
    & \scalea{0.327} & \scalea{0.411}
    & \scalea{0.266} & \scalea{0.364} 
    & \scalea{0.587} & \scalea{0.582} \\
    
    & \scalea{336} 
    & \scalea{\bst{0.179}} & \scalea{\bst{0.270}} 
    & \scalea{\subbst{0.179}} & \scalea{\subbst{0.272}} 
    & \scalea{0.207} & \scalea{0.296} 
    & \scalea{0.208} & \scalea{0.304} 
    & \scalea{0.249} & \scalea{0.344} 
    & \scalea{0.223} & \scalea{0.319} 
    & \scalea{0.226} & \scalea{0.340} 
    & \scalea{0.232} & \scalea{0.343} 
    & \scalea{0.354} & \scalea{0.434}
    & \scalea{0.292} & \scalea{0.383} 
    & \scalea{0.590} & \scalea{0.588} \\
    
    & \scalea{720} 
    & \scalea{\bst{0.208}} & \scalea{\bst{0.297}} 
    & \scalea{\subbst{0.209}} & \scalea{\subbst{0.298}} 
    & \scalea{0.245} & \scalea{0.332}
    & \scalea{0.289} & \scalea{0.363} 
    & \scalea{0.284} & \scalea{0.373} 
    & \scalea{0.258} & \scalea{0.350} 
    & \scalea{0.282} & \scalea{0.379} 
    & \scalea{0.268} & \scalea{0.371} 
    & \scalea{0.343} & \scalea{0.423}
    & \scalea{0.287} & \scalea{0.371} 
    & \scalea{0.602} & \scalea{0.601} \\
    \cmidrule(lr){2-24}
    
    & \scalea{Avg} 
    & \scalea{\bst{0.174}} & \scalea{\bst{0.265}} 
    & \scalea{\subbst{0.176}} & \scalea{\subbst{0.267}} 
    & \scalea{0.209} & \scalea{0.297} 
    & \scalea{0.214} & \scalea{0.307} 
    & \scalea{0.251} & \scalea{0.344} 
    & \scalea{0.225} & \scalea{0.319} 
    & \scalea{0.229} & \scalea{0.339} 
    & \scalea{0.228} & \scalea{0.339} 
    & \scalea{0.335} & \scalea{0.416}
    & \scalea{0.274} & \scalea{0.367} 
    & \scalea{0.617} & \scalea{0.598} \\
    \midrule

    \multirow{5}{*}{{\rotatebox{90}{\scalebox{0.95}{Traffic}}}} 
    & \scalea{96} 
    & \scalea{\bst{0.396}} & \scalea{\bst{0.271}}  
    & \scalea{\subbst{0.397}} & \scalea{\subbst{0.272}} 
    & \scalea{0.528} & \scalea{0.341} 
    & \scalea{0.504} & \scalea{0.298} 
    & \scalea{0.805} & \scalea{0.493} 
    & \scalea{0.697} & \scalea{0.429} 
    & \scalea{0.577} & \scalea{0.362} 
    & \scalea{0.609} & \scalea{0.385} 
    & \scalea{0.698} & \scalea{0.390}
    & \scalea{0.686} & \scalea{0.385} 
    & \scalea{1.451} & \scalea{0.744} \\
    
    & \scalea{192} 
    & \scalea{\bst{0.416}} & \scalea{\bst{0.278}} 
    & \scalea{\subbst{0.418}} & \scalea{\subbst{0.279}} 
    & \scalea{0.531} & \scalea{0.338} 
    & \scalea{0.526} & \scalea{0.305} 
    & \scalea{0.756} & \scalea{0.474} 
    & \scalea{0.647} & \scalea{0.407} 
    & \scalea{0.603} & \scalea{0.372} 
    & \scalea{0.633} & \scalea{0.400} 
    & \scalea{0.697} & \scalea{0.386}
    & \scalea{0.679} & \scalea{0.377} 
    & \scalea{0.842} & \scalea{0.622} \\
    
    & \scalea{336} 
    & \scalea{\bst{0.425}} & \scalea{\bst{0.283}} 
    & \scalea{\subbst{0.432}} & \scalea{\subbst{0.286}} 
    & \scalea{0.551} & \scalea{0.345} 
    & \scalea{0.540} & \scalea{0.310} 
    & \scalea{0.762} & \scalea{0.477} 
    & \scalea{0.653} & \scalea{0.410} 
    & \scalea{0.615} & \scalea{0.378} 
    & \scalea{0.637} & \scalea{0.398} 
    & \scalea{0.715} & \scalea{0.397}
    & \scalea{0.663} & \scalea{0.361} 
    & \scalea{0.844} & \scalea{0.620} \\
    
    & \scalea{720} 
    & \scalea{\bst{0.464}} & \scalea{\bst{0.304}} 
    & \scalea{\subbst{0.467}} & \scalea{\subbst{0.305}} 
    & \scalea{0.598} & \scalea{0.367} 
    & \scalea{0.570} & \scalea{0.324} 
    & \scalea{0.719} & \scalea{0.449} 
    & \scalea{0.694} & \scalea{0.429} 
    & \scalea{0.649} & \scalea{0.403} 
    & \scalea{0.668} & \scalea{0.415} 
    & \scalea{0.797} & \scalea{0.443}
    & \scalea{0.693} & \scalea{0.381} 
    & \scalea{0.867} & \scalea{0.624} \\
    \cmidrule(lr){2-24}
    
    & \scalea{Avg} 
    & \scalea{\bst{0.425}} & \scalea{\bst{0.284}} 
    & \scalea{\subbst{0.428}} & \scalea{\subbst{0.286}} 
    & \scalea{0.552} & \scalea{0.348} 
    & \scalea{0.535} & \scalea{0.309} 
    & \scalea{0.760} & \scalea{0.473} 
    & \scalea{0.673} & \scalea{0.419} 
    & \scalea{0.611} & \scalea{0.379} 
    & \scalea{0.637} & \scalea{0.399} 
    & \scalea{0.727} & \scalea{0.404}
    & \scalea{0.680} & \scalea{0.376} 
    & \scalea{1.001} & \scalea{0.652} \\
    \midrule
    
    \multirow{5}{*}{{\rotatebox{90}{\scalebox{0.95}{Weather}}}}
    & \scalea{96}  
    & \scalea{\bst{0.173}} & \scalea{\bst{0.211}} 
    & \scalea{0.201} & \scalea{0.247} 
    & \scalea{0.184} & \scalea{0.239}
    & \scalea{\subbst{0.178}} & \scalea{\subbst{0.226}} 
    & \scalea{0.202} & \scalea{0.261} 
    & \scalea{0.197} & \scalea{0.259} 
    & \scalea{0.221} & \scalea{0.304} 
    & \scalea{0.284} & \scalea{0.355} 
    & \scalea{0.383} & \scalea{0.438}
    & \scalea{0.332} & \scalea{0.383} 
    & \scalea{0.610} & \scalea{0.568} \\
    
    & \scalea{192} 
    & \scalea{0.246} & \scalea{0.280}
    & \scalea{0.250} & \scalea{0.283} 
    & \scalea{\bst{0.223}} & \scalea{\subbst{0.275}} 
    & \scalea{\subbst{0.227}} & \scalea{\bst{0.266}} 
    & \scalea{0.242} & \scalea{0.298} 
    & \scalea{0.236} & \scalea{0.294} 
    & \scalea{0.275} & \scalea{0.345} 
    & \scalea{0.313} & \scalea{0.371} 
    & \scalea{0.415} & \scalea{0.449}
    & \scalea{0.634} & \scalea{0.539} 
    & \scalea{0.541} & \scalea{0.552} \\
    
    & \scalea{336} 
    & \scalea{\subbst{0.277}} & \scalea{\bst{0.296}} 
    & \scalea{0.302} & \scalea{0.317} 
    & \scalea{\bst{0.272}} & \scalea{0.316} 
    & \scalea{0.283} & \scalea{\subbst{0.305}} 
    & \scalea{0.287} & \scalea{0.335} 
    & \scalea{0.282} & \scalea{0.332} 
    & \scalea{0.338} & \scalea{0.379} 
    & \scalea{0.359} & \scalea{0.393} 
    & \scalea{0.618} & \scalea{0.551}
    & \scalea{0.656} & \scalea{0.579} 
    & \scalea{0.565} & \scalea{0.569} \\
    
    & \scalea{720} 
    & \scalea{0.367} & \scalea{\subbst{0.356}} 
    & \scalea{0.370} & \scalea{0.362} 
    & \scalea{\bst{0.340}} & \scalea{0.363} 
    & \scalea{0.359} & \scalea{\bst{0.355}} 
    & \scalea{0.351} & \scalea{0.386} 
    & \scalea{\subbst{0.347}} & \scalea{0.384} 
    & \scalea{0.408} & \scalea{0.418} 
    & \scalea{0.440} & \scalea{0.446} 
    & \scalea{0.963} & \scalea{0.726}
    & \scalea{0.908} & \scalea{0.706} 
    & \scalea{0.622} & \scalea{0.601} \\
    \cmidrule(lr){2-24}
    
    & \scalea{Avg}  
    & \scalea{0.266} & \scalea{\bst{0.286}} 
    & \scalea{0.281} & \scalea{0.302} 
    & \scalea{\bst{0.255}} & \scalea{0.299} 
    & \scalea{\subbst{0.262}} & \scalea{\subbst{0.288}} 
    & \scalea{0.271} & \scalea{0.320} 
    & \scalea{0.265} & \scalea{0.317} 
    & \scalea{0.311} & \scalea{0.361} 
    & \scalea{0.349} & \scalea{0.391} 
    & \scalea{0.595} & \scalea{0.541}
    & \scalea{0.632} & \scalea{0.552} 
    & \scalea{0.584} & \scalea{0.572} \\
    \midrule
    
    \multirow{5}{*}{{\rotatebox{90}{\scalebox{0.95}{PEMS03}}}}
    & \scalea{12}  
    & \scalea{\bst{0.067}} & \scalea{\bst{0.173}} 
    & \scalea{\subbst{0.069}} & \scalea{\subbst{0.175}} 
    & \scalea{0.083} & \scalea{0.194} 
    & \scalea{0.082} & \scalea{0.188} 
    & \scalea{0.117} & \scalea{0.225} 
    & \scalea{0.122} & \scalea{0.245} 
    & \scalea{0.123} & \scalea{0.248} 
    & \scalea{0.239} & \scalea{0.365} 
    & \scalea{0.122} & \scalea{0.226}
    & \scalea{0.107} & \scalea{0.209}
    & \scalea{0.632} & \scalea{0.606} \\
    
    & \scalea{24}  
    & \scalea{\bst{0.095}} & \scalea{\bst{0.205}} 
    & \scalea{\subbst{0.098}} & \scalea{\subbst{0.210}} 
    & \scalea{0.127} & \scalea{0.241} 
    & \scalea{0.110} & \scalea{0.216} 
    & \scalea{0.233} & \scalea{0.320} 
    & \scalea{0.202} & \scalea{0.320} 
    & \scalea{0.160} & \scalea{0.287} 
    & \scalea{0.492} & \scalea{0.506} 
    & \scalea{0.129} & \scalea{0.233}
    & \scalea{0.121} & \scalea{0.227} 
    & \scalea{0.655} & \scalea{0.626} \\
    
    & \scalea{36}  
    & \scalea{\bst{0.127}} & \scalea{\subbst{0.240}} 
    & \scalea{\subbst{0.131}} & \scalea{0.243} 
    & \scalea{0.169} & \scalea{0.281} 
    & \scalea{0.133} & \scalea{\bst{0.236}} 
    & \scalea{0.380} & \scalea{0.422} 
    & \scalea{0.275} & \scalea{0.382} 
    & \scalea{0.191} & \scalea{0.321} 
    & \scalea{0.399} & \scalea{0.459} 
    & \scalea{0.143} & \scalea{0.249}
    & \scalea{0.133} & \scalea{0.243} 
    & \scalea{0.678} & \scalea{0.644} \\
    
    & \scalea{48}  
    & \scalea{\subbst{0.159}} & \scalea{\subbst{0.270}} 
    & \scalea{0.164} & \scalea{0.275} 
    & \scalea{0.204} & \scalea{0.311} 
    & \scalea{\bst{0.146}} & \scalea{\bst{0.251}} 
    & \scalea{0.536} & \scalea{0.511} 
    & \scalea{0.335} & \scalea{0.429} 
    & \scalea{0.223} & \scalea{0.350} 
    & \scalea{0.875} & \scalea{0.723} 
    & \scalea{0.153} & \scalea{0.255}
    & \scalea{0.144} & \scalea{0.253} 
    & \scalea{0.699} & \scalea{0.659} \\
    \cmidrule(lr){2-24}
    
    & \scalea{Avg}  
    & \scalea{\bst{0.112}} & \scalea{\bst{0.222}} 
    & \scalea{\subbst{0.116}} & \scalea{0.226} 
    & \scalea{0.146} & \scalea{0.257} 
    & \scalea{0.118} & \scalea{\subbst{0.223}} 
    & \scalea{0.316} & \scalea{0.370} 
    & \scalea{0.233} & \scalea{0.344} 
    & \scalea{0.174} & \scalea{0.302} 
    & \scalea{0.501} & \scalea{0.513} 
    & \scalea{0.137} & \scalea{0.241}
    & \scalea{0.126} & \scalea{0.233} 
    & \scalea{0.666} & \scalea{0.634} \\
    \midrule

    \multirow{5}{*}{{\rotatebox{90}{\scalebox{0.95}{PEMS08}}}}
    & \scalea{12}  
    & \scalea{\bst{0.079}} & \scalea{\bst{0.181}} 
    & \scalea{\subbst{0.085}} & \scalea{\subbst{0.189}} 
    & \scalea{0.095} & \scalea{0.204} 
    & \scalea{0.110} & \scalea{0.209} 
    & \scalea{0.121} & \scalea{0.231} 
    & \scalea{0.152} & \scalea{0.274} 
    & \scalea{0.175} & \scalea{0.275} 
    & \scalea{0.446} & \scalea{0.483} 
    & \scalea{0.268} & \scalea{0.281}
    & \scalea{0.213} & \scalea{0.236} 
    & \scalea{0.680} & \scalea{0.607} \\
    
    & \scalea{24}  
    & \scalea{\bst{0.114}} & \scalea{\bst{0.218}} 
    & \scalea{\subbst{0.131}} & \scalea{\subbst{0.236}} 
    & \scalea{0.150} & \scalea{0.259} 
    & \scalea{0.142} & \scalea{0.239} 
    & \scalea{0.232} & \scalea{0.326} 
    & \scalea{0.245} & \scalea{0.350} 
    & \scalea{0.211} & \scalea{0.305} 
    & \scalea{0.488} & \scalea{0.509} 
    & \scalea{0.296} & \scalea{0.302}
    & \scalea{0.238} & \scalea{0.256} 
    & \scalea{0.701} & \scalea{0.622} \\
    
    & \scalea{36}  
    & \scalea{\bst{0.158}} & \scalea{\bst{0.256}} 
    & \scalea{0.182} & \scalea{0.282} 
    & \scalea{0.202} & \scalea{0.305} 
    & \scalea{\subbst{0.167}} & \scalea{\subbst{0.258}} 
    & \scalea{0.379} & \scalea{0.428} 
    & \scalea{0.344} & \scalea{0.417} 
    & \scalea{0.250} & \scalea{0.338} 
    & \scalea{0.532} & \scalea{0.513} 
    & \scalea{0.340} & \scalea{0.327}
    & \scalea{0.263} & \scalea{0.277} 
    & \scalea{0.727} & \scalea{0.637} \\
    
    & \scalea{48}  
    & \scalea{\subbst{0.203}} & \scalea{\subbst{0.290}} 
    & \scalea{0.236} & \scalea{0.323} 
    & \scalea{0.250} & \scalea{0.341} 
    & \scalea{\bst{0.195}} & \scalea{\bst{0.274}} 
    & \scalea{0.543} & \scalea{0.527} 
    & \scalea{0.437} & \scalea{0.469} 
    & \scalea{0.293} & \scalea{0.371} 
    & \scalea{1.052} & \scalea{0.781} 
    & \scalea{0.373} & \scalea{0.345}
    & \scalea{0.283} & \scalea{0.295} 
    & \scalea{0.746} & \scalea{0.648} \\
    \cmidrule(lr){2-24}
    
    & \scalea{Avg}  
    & \scalea{\bst{0.138}} & \scalea{\bst{0.236}} 
    & \scalea{0.159} & \scalea{0.258} 
    & \scalea{0.174} & \scalea{0.277} 
    & \scalea{\subbst{0.154}} & \scalea{\subbst{0.245}} 
    & \scalea{0.319} & \scalea{0.378} 
    & \scalea{0.294} & \scalea{0.377} 
    & \scalea{0.232} & \scalea{0.322} 
    & \scalea{0.630} & \scalea{0.572} 
    & \scalea{0.319} & \scalea{0.314}
    & \scalea{0.249} & \scalea{0.266} 
    & \scalea{0.713} & \scalea{0.629} \\
    \midrule
    
    \multicolumn{2}{c|}{\scalea{{$1^{\text{st}}$ Count}}} 
    & \scalea{\bst{33}} & \scalea{\bst{36}} 
    & \scalea{0} & \scalea{0} 
    & \scalea{\subbst{4}} & \scalea{0} 
    & \scalea{2} & \scalea{\subbst{6}} 
    & \scalea{0} & \scalea{0} 
    & \scalea{1} & \scalea{3} 
    & \scalea{\subbst{4}} & \scalea{0} 
    & \scalea{1} & \scalea{0} 
    & \scalea{0} & \scalea{0} 
    & \scalea{0} & \scalea{0} 
    & \scalea{0} & \scalea{0} \\
    \bottomrule
  \end{tabular}
\end{table}

We present a comprehensive comparison of the multi-step forecasting task in Table~\ref{tab:multistep_app_full}. The iTransformer model serves as the base model for implementing the MoLA paradigm. Despite iTransformer's initial performance gap compared to other state-of-the-art baseline models, integrating MoLA significantly enhances its forecasting accuracy. 

Specifically, MoLA achieves the lowest MSE in 33 out of 45 cases and the lowest MAE in 36 out of 45 cases, illustrating its effectiveness in enhancing model performance. The improvements brought by MoLA are particularly noticeable on challenging datasets such as ETTm1 and Traffic, where capturing long-term dependencies is crucial. These results underscore the robustness and adaptability of the MoLA paradigm.
While there are a few instances where MoLA does not achieve the top performance, this can be attributed to the inherent advantages of specific models in particular contexts. For example, FreTS shows competitive results on the Weather dataset, where its architectural design may better suit certain meteorological patterns. Nonetheless, the overall performance of MoLA demonstrates its strength as a fine-tuning framework for time-series forecasting, consistently mitigating the expressiveness bottleneck and delivering superior results across a wide range of forecasting tasks.

\subsection{Comparative Studies}\label{sec:comparative_app}
\begin{table}
  \caption{Results on the multi-step forecasting task with AR-F paradigm. iTransformer is served as the base and single-step model. \emph{Avg} indicates the results averaged over forecasting lengths: T=96, 192, 336 and 720 for ETT, ECL, Traffic and Weather dataset, T=12,24,36 and 48 for PEMS datasets.}\label{tab:arf}
  \centering
  \begin{threeparttable}
  \renewcommand{\arraystretch}{0.8}
  \setlength{\tabcolsep}{2pt}
  \scriptsize
  \renewcommand{\multirowsetup}{\centering}
  \begin{tabular}{c|c|cc|cc|cc|cc|cc|cc|cc|c|cc|cc}
    \toprule
    \multicolumn{2}{l}{\rotatebox{0}{\scaleb{Datasets}}} & 
    \multicolumn{2}{c}{\rotatebox{0}{\scaleb{ETTm1}}} &
    \multicolumn{2}{c}{\rotatebox{0}{\scaleb{ETTm2}}} &
    \multicolumn{2}{c}{\rotatebox{0}{\scaleb{ETTh1}}} &
    \multicolumn{2}{c}{\rotatebox{0}{\scaleb{ETTh2}}} &
    \multicolumn{2}{c}{\rotatebox{0}{\scaleb{ECL}}} &
    \multicolumn{2}{c}{\rotatebox{0}{\scaleb{Traffic}}} &
    \multicolumn{2}{c}{\rotatebox{0}{\scaleb{Weather}}} &
    \multicolumn{1}{c}{} &
    \multicolumn{2}{c}{\rotatebox{0}{\scaleb{PEMS03}}} &
    \multicolumn{2}{c}{\rotatebox{0}{\scaleb{PEMS08}}} \\
    \cmidrule(lr){3-4} \cmidrule(lr){5-6}\cmidrule(lr){7-8} \cmidrule(lr){9-10}\cmidrule(lr){11-12} \cmidrule(lr){13-14} \cmidrule(lr){15-16} \cmidrule(lr){18-19} \cmidrule(lr){20-21}
    \multicolumn{2}{l}{\rotatebox{0}{\scaleb{Metrics}}}  & \scalea{MSE} & \scalea{MAE}  & \scalea{MSE} & \scalea{MAE}  & \scalea{MSE} & \scalea{MAE}  & \scalea{MSE} & \scalea{MAE}  & \scalea{MSE} & \scalea{MAE}  & \scalea{MSE} & \scalea{MAE} & \scalea{MSE} & \multicolumn{1}{c}{\scalea{MAE}} & \multicolumn{1}{c}{} & \scalea{MSE} & \scalea{MAE} & \scalea{MSE} & \scalea{MAE} \\
    \midrule

    \multirow{5}{*}{{\rotatebox{90}{\scalebox{0.95}{AR-F}}}}
    & \scalea{96}
    & \scalea{0.371} & \scalea{0.385}
    & \scalea{0.200} & \scalea{0.278}
    & \scalea{0.454} & \scalea{0.441}
    & \scalea{0.339} & \scalea{0.368}
    & \scalea{0.307} & \scalea{0.360}
    & \scalea{0.590} & \scalea{0.364}
    & \scalea{0.192} & \scalea{0.234}
    & \scalea{12}
    & \scalea{0.089} & \scalea{0.204}
    & \scalea{0.088} & \scalea{0.195} \\
    
    & \scalea{192} 
    & \scalea{0.446} & \scalea{0.425} 
    & \scalea{0.262} & \scalea{0.315} 
    & \scalea{0.525} & \scalea{0.484} 
    & \scalea{0.418} & \scalea{0.413} 
    & \scalea{0.396} & \scalea{0.424} 
    & \scalea{0.603} & \scalea{0.385} 
    & \scalea{0.240} & \scalea{0.273} 
    & \scalea{24}
    & \scalea{0.177} & \scalea{0.287} 
    & \scalea{0.262} & \scalea{0.337} \\
    
    & \scalea{336} 
    & \scalea{0.493} & \scalea{0.461} 
    & \scalea{0.463} & \scalea{0.443} 
    & \scalea{0.580} & \scalea{0.518} 
    & \scalea{0.456} & \scalea{0.449} 
    & \scalea{0.550} & \scalea{0.513} 
    & \scalea{0.861} & \scalea{0.495} 
    & \scalea{0.368} & \scalea{0.357} 
    & \scalea{36}
    & \scalea{0.315} & \scalea{0.382} 
    & \scalea{0.562} & \scalea{0.486} \\
    
    & \scalea{720} 
    & \scalea{0.629} & \scalea{0.520} 
    & \scalea{0.430} & \scalea{0.418} 
    & \scalea{0.620} & \scalea{0.550} 
    & \scalea{0.448} & \scalea{0.458} 
    & \scalea{0.978} & \scalea{0.706} 
    & \scalea{1.341} & \scalea{0.588} 
    & \scalea{0.600} & \scalea{0.444} 
    & \scalea{48}
    & \scalea{0.515} & \scalea{0.487} 
    & \scalea{1.133} & \scalea{0.674} \\
    \cmidrule(lr){2-21}
    
    & \scalea{Avg} 
    & \scalea{0.485} & \scalea{0.448} 
    & \scalea{0.339} & \scalea{0.363} 
    & \scalea{0.545} & \scalea{0.498} 
    & \scalea{0.415} & \scalea{0.422} 
    & \scalea{0.558} & \scalea{0.501} 
    & \scalea{0.849} & \scalea{0.458} 
    & \scalea{0.350} & \scalea{0.327} 
    & \scalea{Avg}
    & \scalea{0.274} & \scalea{0.340} 
    & \scalea{0.511} & \scalea{0.423} \\
    \bottomrule
  \end{tabular}
  \end{threeparttable}
\end{table}

We present a detailed comparison of the multi-step forecasting task using the AR-F paradigm in Table~\ref{tab:arf}, where iTransformer serves as both the base and single-step model. Overall, the AR-F paradigm demonstrates a significant issue with error accumulation as the forecasting horizon increases. This effect is particularly noticeable in datasets like PEMS08, where, for example, the MSE dramatically increases as the forecasting length extends from 12 to 48 time steps. 

This behavior underscores the limitations of the AR-F paradigm for long-term forecasting tasks. As the model is required to recursively predict future values based on previous predictions, the compounding effect of errors leads to substantially degraded performance over longer horizons. This observation highlights the need for more robust solutions, such as our MoLA paradigm, which mitigates the error accumulation problem in AR-F and expressiveness bottleneck in MT-F by adopting a more modularized and horizon-specific fine-tuning approach, allowing for more stable and accurate long-term forecasts.

\subsection{Generalization Studies}
\label{apdx:adapt}

\begin{wrapfigure}{r}{0.5\linewidth}
\vspace{-6mm}
\begin{center}
\includegraphics[width=\linewidth]{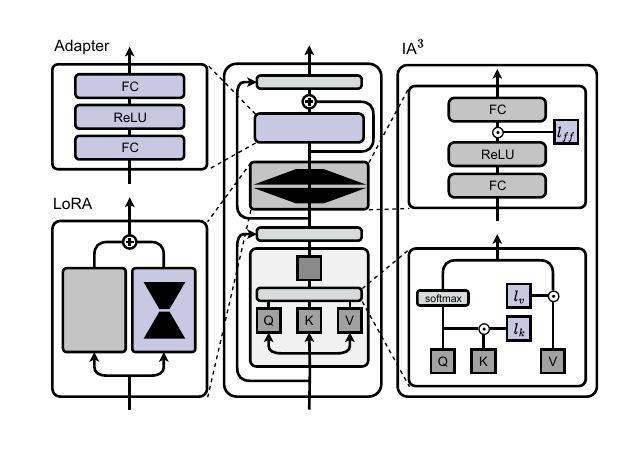}
\caption{Visualization of common parameter-efficient fine-tuning strategies.}
\label{fig:peft}
\end{center}
\vspace{-4mm}
\end{wrapfigure}

We visualize the three prominent PEFT methods: Adapter, LoRA and IA$^3$ in Figure~\ref{fig:peft}, showing the architectural modifications introduced by each PEFT technique within a typical transformer block.
Adapter, shown in the upper left, introduces additional fully connected (FC) layers and a short-cut connection after the FFN layer. The Adapter module typically consists of a down-projection FC layer, followed by a non-linearity (often ReLU), an up-projection FC layer, and a residual connection. This approach provides a compact, trainable module that can adjust the model's behavior for specific tasks without modifying the entire network.
LoRA, illustrated on the lower left, modifies the FFN layer by adding low-rank matrices to the frozen weight matrix. LoRA decomposes the weight update into two low-rank matrices, enabling the model to learn task-specific adaptations in a parameter-efficient manner.
IA$^3$, shown on the right, works by applying learnable scaling factors to the key and value projections in the self-attention mechanism and to the hidden representations in the feed-forward network, allowing for fine-grained control over the model's behavior with minimal additional parameters.

For a fair comparison, we matched the low-rank configuration of the Adapter and the LoRA modules with that of the LoRA expert in MoLA, applying both to the feed-forward network (FFN) layers. IA$^3$, on the other hand, dynamically adjusts the weights of intermediate hidden vectors in both the FFN and attention layers.

\subsection{Hyperparameter Sensitivity}
\label{apdx:sensi}
In this section, we further analyze the sensitivity of MoLA to three key hyperparameters under different segment number settings: the low-rank value $r$ in the LoRA modules, the learning rate $\eta$, and the number of LoRA expert modules $\mP$. The experiments are conducted using iTransformer across two datasets, ETTh1 and Weather, with the results visualized in Figure~\ref{fig:sensi-itrans-etth1} and Figure~\ref{fig:sensi-itrans-weather}.

The results generally indicate that both extremely low and high values of rank $r$ and learning rate $\eta$ lead to performance degradation. This pattern suggests that overly high ranks may lead to overfitting, while excessively low ranks may not provide sufficient flexibility for effective model adaptation. Similarly, very high learning rates can cause instability in training, while very low learning rates may result in slow convergence or getting stuck in suboptimal solutions.
Interestingly, for the number of expert modules $\mP$, MoLA exhibits high stability as long as $\mP$ is not significantly lower than the number of segments $\K$. This robustness indicates that the model can effectively leverage multiple experts to capture diverse patterns across different forecasting segments.

These findings indicate that fine-tuning the hyperparameters for each specific forecasting horizon can further improve performance. However, our results demonstrate that even without exhaustive tuning, the MoLA paradigm delivers robust performance improvements, highlighting its flexibility and effectiveness in time-series forecasting tasks.

\begin{figure}
\subfigure[Varying $r$ on ETTh1.]{
    \includegraphics[width=0.245\linewidth]{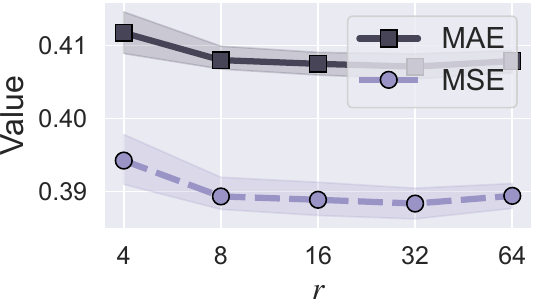}
    \hfill
    \includegraphics[width=0.245\linewidth]{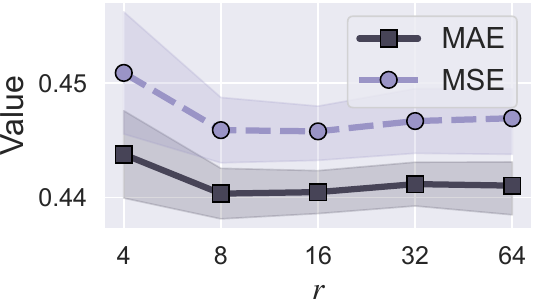}
    \hfill
    \includegraphics[width=0.245\linewidth]{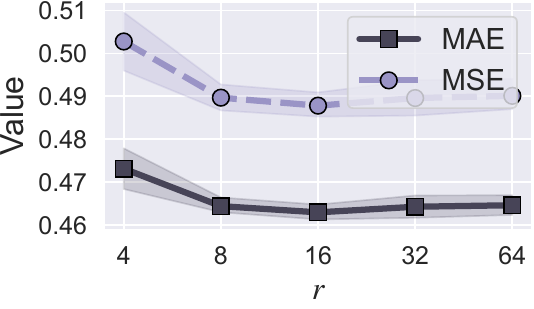}
    \hfill
    \includegraphics[width=0.245\linewidth]{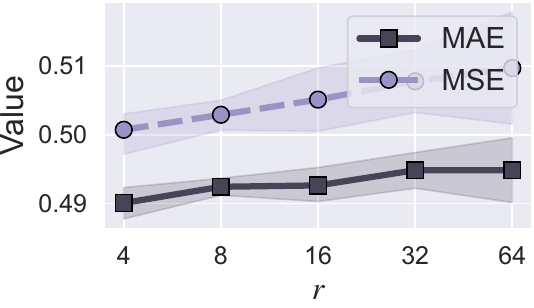}
}

\subfigure[Varying $\eta$ on ETTh1.]{
    \includegraphics[width=0.245\linewidth]{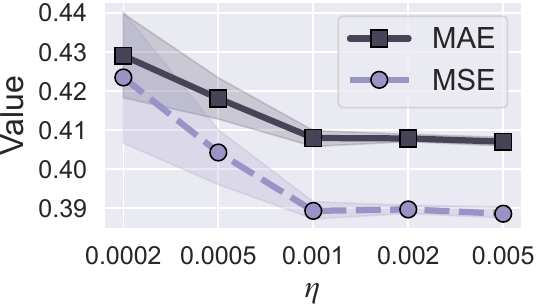}
    \hfill
    \includegraphics[width=0.245\linewidth]{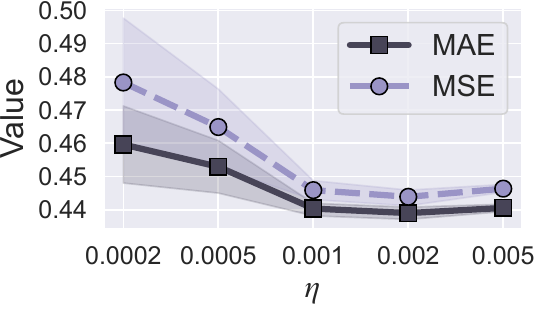}
    \hfill
    \includegraphics[width=0.245\linewidth]{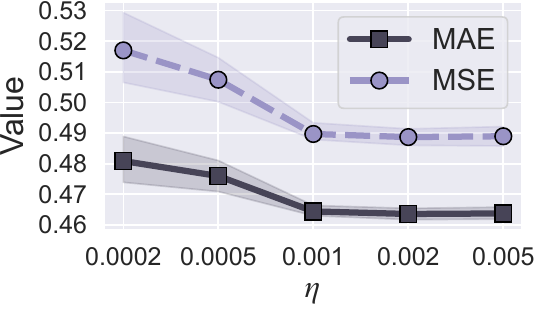}
    \hfill
    \includegraphics[width=0.245\linewidth]{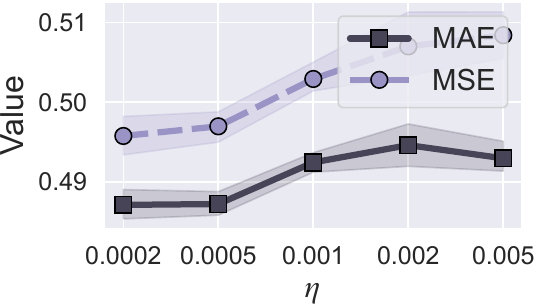}
}

\subfigure[Varying $\mP$ on ETTh1.]{
    \includegraphics[width=0.245\linewidth]{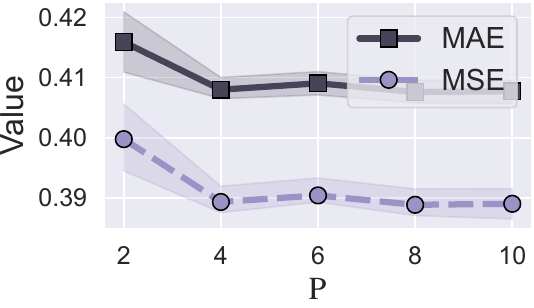}
    \hfill
    \includegraphics[width=0.245\linewidth]{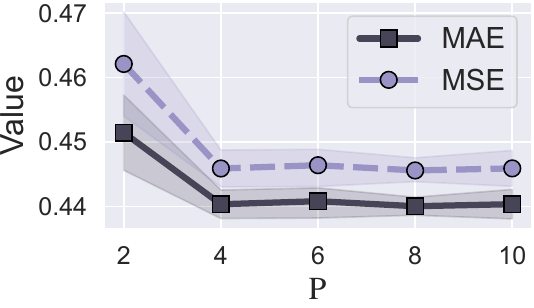}
    \hfill
    \includegraphics[width=0.245\linewidth]{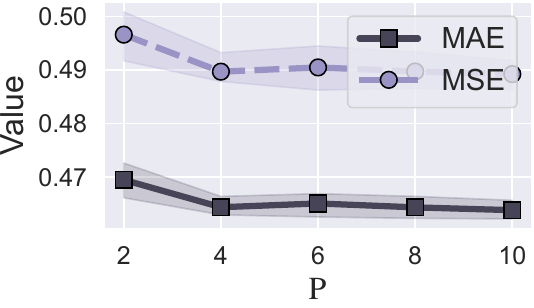}
    \hfill
    \includegraphics[width=0.245\linewidth]{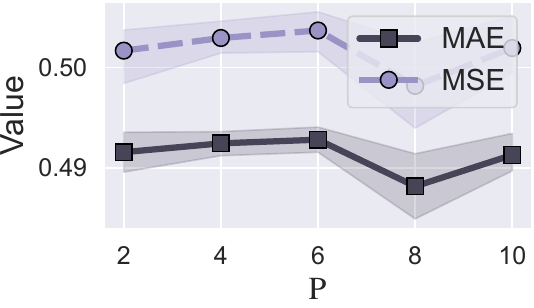}
}
\caption{Performance of iTransformer enhanced by MoLA given different low rank of LoRA modules $r$, learning rate $\eta$ and the number of experts $\mP$. Different columns correspond to different number of forecasting length $\T$ (from left to right: 96, 192, 336, 720). The results are averaged on four forecasting segment number (2, 3, 4, 6) with shaded areas being 50\% confidence intervals.}
\label{fig:sensi-itrans-etth1}
\end{figure}

\begin{figure}
\subfigure[Varying $r$ on Weather.]{
    \includegraphics[width=0.245\linewidth]{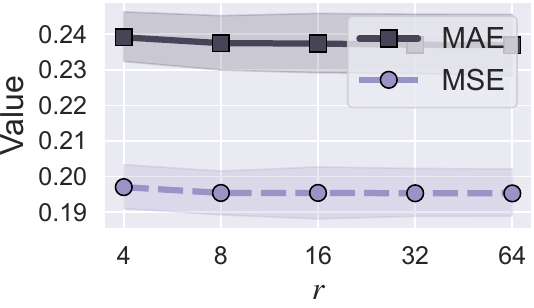}
    \hfill
    \includegraphics[width=0.245\linewidth]{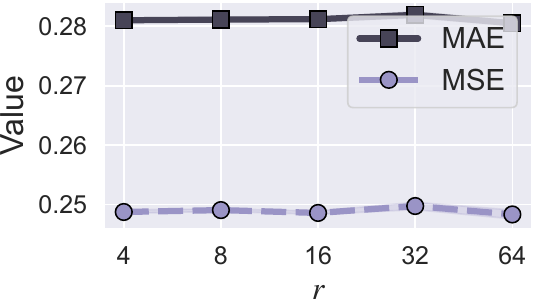}
    \hfill
    \includegraphics[width=0.245\linewidth]{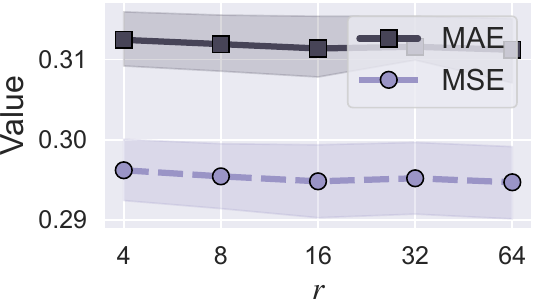}
    \hfill
    \includegraphics[width=0.245\linewidth]{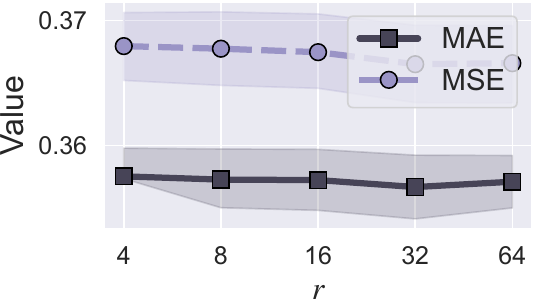}
}

\subfigure[Varying $\eta$ on Weather.]{
    \includegraphics[width=0.245\linewidth]{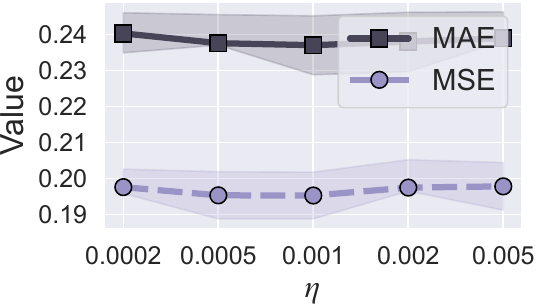}
    \hfill
    \includegraphics[width=0.245\linewidth]{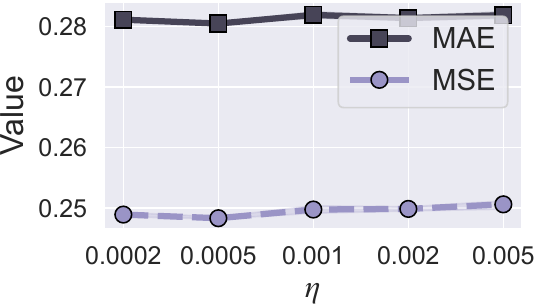}
    \hfill
    \includegraphics[width=0.245\linewidth]{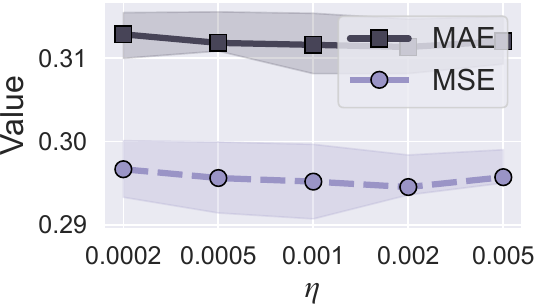}
    \hfill
    \includegraphics[width=0.245\linewidth]{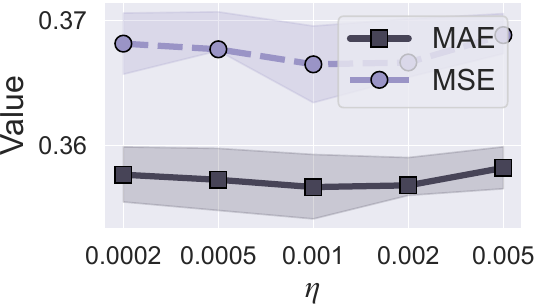}
}

\subfigure[Varying $\mP$ on Weather.]{
    \includegraphics[width=0.245\linewidth]{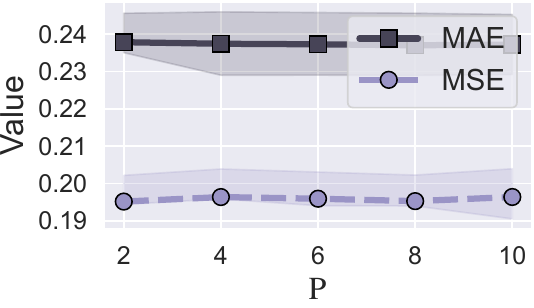}
    \hfill
    \includegraphics[width=0.245\linewidth]{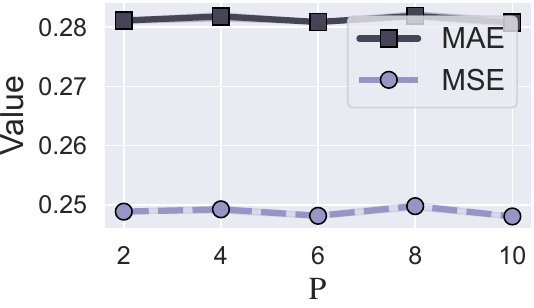}
    \hfill
    \includegraphics[width=0.245\linewidth]{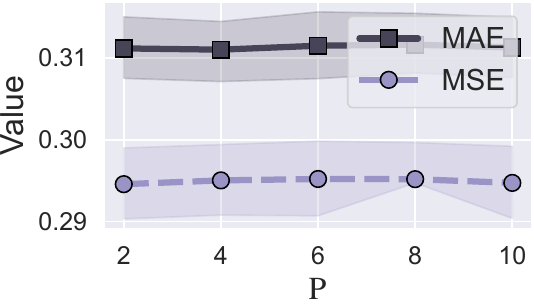}
    \hfill
    \includegraphics[width=0.245\linewidth]{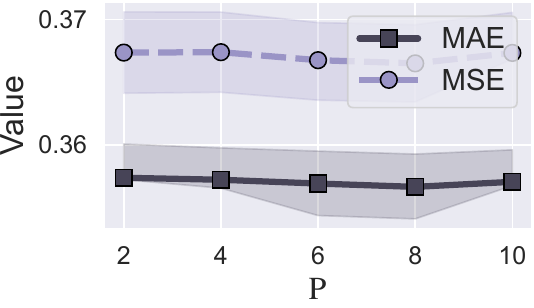}
}
\caption{Performance of iTransformer enhanced by MoLA given different low rank of LoRA modules $r$, learning rate $\eta$ and the number of experts $\mP$. Different columns correspond to different number of forecasting length $\T$ (from left to right: 96, 192, 336, 720). The results are averaged on four forecasting segment number (2, 3, 4, 6) with shaded areas being 50\% confidence intervals. }
\label{fig:sensi-itrans-weather}
\end{figure}

\subsection{Extra Discussions}
\label{apdx:extra}

\begin{table}
  \caption{Full results on the multi-step forecasting task with jointly fine-tuning weight and bias under MF paradigm. The length of history window is set to 96 for all baselines. \emph{Avg} indicates the results averaged over forecasting lengths: T=96, 192, 336 and 720 for ETT, ECL, Traffic and Weather dataset, T=12,24,36 and 48 for PEMS datasets.}\label{tab:bias_full}
  \centering
  \begin{threeparttable}
  \renewcommand{\arraystretch}{0.8}
  \setlength{\tabcolsep}{2pt}
  \scriptsize
  \renewcommand{\multirowsetup}{\centering}
  \begin{tabular}{c|c|cc|cc|cc|cc|cc|cc|cc|c|cc|cc}
    \toprule
    \multicolumn{2}{l}{\rotatebox{0}{\scaleb{Datasets}}} & 
    \multicolumn{2}{c}{\rotatebox{0}{\scaleb{ETTm1}}} &
    \multicolumn{2}{c}{\rotatebox{0}{\scaleb{ETTm2}}} &
    \multicolumn{2}{c}{\rotatebox{0}{\scaleb{ETTh1}}} &
    \multicolumn{2}{c}{\rotatebox{0}{\scaleb{ETTh2}}} &
    \multicolumn{2}{c}{\rotatebox{0}{\scaleb{ECL}}} &
    \multicolumn{2}{c}{\rotatebox{0}{\scaleb{Traffic}}} &
    \multicolumn{2}{c}{\rotatebox{0}{\scaleb{Weather}}} &
    \multicolumn{1}{c}{} &
    \multicolumn{2}{c}{\rotatebox{0}{\scaleb{PEMS03}}} &
    \multicolumn{2}{c}{\rotatebox{0}{\scaleb{PEMS08}}} \\
    \cmidrule(lr){3-4} \cmidrule(lr){5-6}\cmidrule(lr){7-8} \cmidrule(lr){9-10}\cmidrule(lr){11-12} \cmidrule(lr){13-14} \cmidrule(lr){15-16} \cmidrule(lr){18-19} \cmidrule(lr){20-21}
    \multicolumn{2}{l}{\rotatebox{0}{\scaleb{Metrics}}}  & \scalea{MSE} & \scalea{MAE}  & \scalea{MSE} & \scalea{MAE}  & \scalea{MSE} & \scalea{MAE}  & \scalea{MSE} & \scalea{MAE}  & \scalea{MSE} & \scalea{MAE}  & \scalea{MSE} & \scalea{MAE} & \scalea{MSE} & \multicolumn{1}{c}{\scalea{MAE}} & \multicolumn{1}{c}{} & \scalea{MSE} & \scalea{MAE} & \scalea{MSE} & \scalea{MAE} \\
    \midrule

    \multirow{5}{*}{{\rotatebox{90}{\scalebox{0.95}{MoLA-WB}}}}
    & \scalea{96}
    & \scalea{0.337} & \scalea{0.368}
    & \scalea{0.182} & \scalea{0.263}
    & \scalea{0.378} & \scalea{0.400}
    & \scalea{0.297} & \scalea{0.345}
    & \scalea{0.147} & \scalea{0.237}
    & \scalea{0.394} & \scalea{0.270}
    & \scalea{0.201} & \scalea{0.245}
    & \scalea{12}
    & \scalea{0.070} & \scalea{0.175}
    & \scalea{0.079} & \scalea{0.181} \\
    
    & \scalea{192} 
    & \scalea{0.382} & \scalea{0.394} 
    & \scalea{0.246} & \scalea{0.306} 
    & \scalea{0.434} & \scalea{0.431} 
    & \scalea{0.379} & \scalea{0.398} 
    & \scalea{0.163} & \scalea{0.254} 
    & \scalea{0.420} & \scalea{0.282} 
    & \scalea{0.246} & \scalea{0.279} 
    & \scalea{24}
    & \scalea{0.095} & \scalea{0.205} 
    & \scalea{0.114} & \scalea{0.218} \\
    
    & \scalea{336} 
    & \scalea{0.419} & \scalea{0.420} 
    & \scalea{0.315} & \scalea{0.350} 
    & \scalea{0.477} & \scalea{0.456} 
    & \scalea{0.420} & \scalea{0.431} 
    & \scalea{0.182} & \scalea{0.273} 
    & \scalea{0.640} & \scalea{0.426} 
    & \scalea{0.280} & \scalea{0.297} 
    & \scalea{36}
    & \scalea{0.127} & \scalea{0.239} 
    & \scalea{0.158} & \scalea{0.256} \\
    
    & \scalea{720} 
    & \scalea{0.483} & \scalea{0.456} 
    & \scalea{0.415} & \scalea{0.407} 
    & \scalea{0.495} & \scalea{0.488} 
    & \scalea{0.432} & \scalea{0.448} 
    & \scalea{0.211} & \scalea{0.300} 
    & \scalea{0.585} & \scalea{0.352} 
    & \scalea{0.367} & \scalea{0.356} 
    & \scalea{48}
    & \scalea{0.158} & \scalea{0.269} 
    & \scalea{0.202} & \scalea{0.290} \\
    \cmidrule(lr){2-21}
    
    & \scalea{Avg} 
    & \scalea{0.405} & \scalea{0.409} 
    & \scalea{0.289} & \scalea{0.332} 
    & \scalea{0.446} & \scalea{0.444} 
    & \scalea{0.382} & \scalea{0.405} 
    & \scalea{0.176} & \scalea{0.266} 
    & \scalea{0.549} & \scalea{0.353} 
    & \scalea{0.274} & \scalea{0.294} 
    & \scalea{Avg}
    & \scalea{0.112} & \scalea{0.222} 
    & \scalea{0.138} & \scalea{0.236} \\
    \bottomrule
  \end{tabular}
  \end{threeparttable}
\end{table}

\paragraph{Discussion on fine-tuning bias.} 
In this section, we extend our analysis by evaluating the impact of jointly fine-tuning both the weight and bias within the LoRA modules under the MoLA paradigm. The results, detailed in Table~\ref{tab:bias_full}, cover a variety of datasets.

Overall, the findings indicate that the joint fine-tuning of weights and biases (MoLA-WB) results in performance degradation compared to fine-tuning weights alone under the MoLA paradigm. Specifically, MoLA-WB tends to increase the risk of overfitting, particularly in datasets with complex temporal dependencies, where the added non-low-rank parameters introduce excessive flexibility, making the model more difficult to optimize effectively. This is especially evident in the larger forecasting horizons, where the performance gap becomes more pronounced.

However, MoLA-WB generally improves the performance of the MT-F paradigm, suggesting that while the joint fine-tuning strategy introduces challenges for the MoLA framework, it may still offer some benefit in simpler or less modular fine-tuning strategies. 
The overall results affirm that focusing on weight fine-tuning alone is a more effective approach for leveraging the full potential of MoLA, ensuring better generalization and predictive accuracy across diverse forecasting tasks.

\paragraph{Discussion on Layer-wise Fine-tuning.}
In this section, we explore the impact of layer-wise fine-tuning in the multi-layer iTransformer model across various datasets. Specifically, for the ETT dataset, we adaptd the first (MoLA-1) and second (MoLA-2) layers, while for the ECL, Traffic, Weather, and PEMS datasets, we extended the study to include the third layer (MoLA-3). The results are detailed in Tables~\ref{tab:layer1} and Table~\ref{tab:layer2}.

From the experiments, we observe that layer-wise fine-tuning generally yields positive performance improvements over the MT-F paradigm in most cases, particularly for datasets like ETTm1, ETTm2, Weather, and ECL. In these datasets, fine-tuning a single layer was sufficient to capture the temporal dependencies effectively, leading to reduced MSE and MAE. For instance, in the ETT datasets, both MoLA-1 and MoLA-2 show competitive results compared to full-layer fine-tuning, indicating that focusing on specific layers can provide significant computational savings without sacrificing accuracy.

However, for more complex datasets such as Traffic and PEMS03, fine-tuning a single layer did not achieve results better than those obtained with the MT-F paradigm. This could be due to the disruption of key inter-layer interactions that are crucial for the hierarchical processing in Transformer-based models. These interactions are particularly important in datasets with complex temporal patterns or multiple variates, where adjustments in a single layer may not provide enough capacity to adapt to the nuances of the time-series data.

The results suggest that while layer-wise fine-tuning can be beneficial in reducing the computational overhead and maintaining high performance, it is dataset-dependent. In datasets with more complex structures, a more comprehensive fine-tuning strategy involving multiple layers or full-layer fine-tuning may be necessary to avoid underfitting and fully capture the temporal dependencies in the data.

\begin{table}
  \caption{Full results on the multi-step forecasting task with layer-wise fine-tuning under MoLA paradigm.}\label{tab:layer1}
  \centering
  \begin{threeparttable}
  \renewcommand{\arraystretch}{0.8} \setlength{\tabcolsep}{5.5pt} \scriptsize
  \renewcommand{\multirowsetup}{\centering}
  \begin{tabular}{c|c|cc|cc|cc|cc}
    \toprule
    \multicolumn{2}{l}{\rotatebox{0}{\scaleb{Datasets}}} & 
    \multicolumn{2}{c}{\rotatebox{0}{\scaleb{ETTm1}}} &
    \multicolumn{2}{c}{\rotatebox{0}{\scaleb{ETTm2}}} &
    \multicolumn{2}{c}{\rotatebox{0}{\scaleb{ETTh1}}} &
    \multicolumn{2}{c}{\rotatebox{0}{\scaleb{ETTh2}}} \\
    \cmidrule(lr){3-4} \cmidrule(lr){5-6}\cmidrule(lr){7-8} \cmidrule(lr){9-10}
    \multicolumn{2}{l}{\rotatebox{0}{\scaleb{Metrics}}}  & \scalea{MSE} & \scalea{MAE}  & \scalea{MSE} & \scalea{MAE}  & \scalea{MSE} & \scalea{MAE}  & \scalea{MSE} & \scalea{MAE} \\
    \midrule

    \multirow{5}{*}{{\rotatebox{90}{\scalebox{0.95}{MoLA-1}}}}
    & \scalea{96}
    & \scalea{0.342} & \scalea{0.373}
    & \scalea{0.183} & \scalea{0.263}
    & \scalea{0.379} & \scalea{0.400}
    & \scalea{0.293} & \scalea{0.342} \\
    
    & \scalea{192} 
    & \scalea{0.386} & \scalea{0.396} 
    & \scalea{0.251} & \scalea{0.310} 
    & \scalea{0.436} & \scalea{0.432} 
    & \scalea{0.380} & \scalea{0.397} \\
    
    & \scalea{336} 
    & \scalea{0.426} & \scalea{0.424} 
    & \scalea{0.315} & \scalea{0.352} 
    & \scalea{0.477} & \scalea{0.456} 
    & \scalea{0.428} & \scalea{0.434}  \\
    
    & \scalea{720} 
    & \scalea{0.495} & \scalea{0.462} 
    & \scalea{0.412} & \scalea{0.406} 
    & \scalea{0.486} & \scalea{0.481} 
    & \scalea{0.428} & \scalea{0.445}  \\
    \cmidrule(lr){2-10}
    
    & \scalea{Avg} 
    & \scalea{0.412} & \scalea{0.414} 
    & \scalea{0.290} & \scalea{0.333} 
    & \scalea{0.444} & \scalea{0.442} 
    & \scalea{0.382} & \scalea{0.404} \\
    \midrule
    
    \multirow{5}{*}{{\rotatebox{90}{\scalebox{0.95}{MoLA-2}}}}
    & \scalea{96}
    & \scalea{0.335} & \scalea{0.369}
    & \scalea{0.181} & \scalea{0.262}
    & \scalea{0.377} & \scalea{0.398}
    & \scalea{0.296} & \scalea{0.343} \\
    
    & \scalea{192} 
    & \scalea{0.382} & \scalea{0.395} 
    & \scalea{0.246} & \scalea{0.306} 
    & \scalea{0.434} & \scalea{0.431} 
    & \scalea{0.384} & \scalea{0.402} \\
    
    & \scalea{336} 
    & \scalea{0.414} & \scalea{0.418} 
    & \scalea{0.308} & \scalea{0.344} 
    & \scalea{0.480} & \scalea{0.459} 
    & \scalea{0.432} & \scalea{0.438}  \\
    
    & \scalea{720} 
    & \scalea{0.480} & \scalea{0.454} 
    & \scalea{0.411} & \scalea{0.405} 
    & \scalea{0.498} & \scalea{0.490} 
    & \scalea{0.430} & \scalea{0.445}  \\
    \cmidrule(lr){2-10}
    
    & \scalea{Avg} 
    & \scalea{0.403} & \scalea{0.409} 
    & \scalea{0.286} & \scalea{0.329} 
    & \scalea{0.447} & \scalea{0.445} 
    & \scalea{0.385} & \scalea{0.407} \\
    \bottomrule
  \end{tabular}
  \end{threeparttable}
\end{table}

\begin{table}
  \caption{Full results on the multi-step forecasting task with layer-wise fine-tuning under MF paradigm.}\label{tab:layer2}
  \centering
  \begin{threeparttable}
  \renewcommand{\arraystretch}{0.8}
  \setlength{\tabcolsep}{5.5pt}
  \scriptsize
  \renewcommand{\multirowsetup}{\centering}
  \begin{tabular}{c|c|cc|cc|cc|c|cc|cc}
    \toprule
    \multicolumn{2}{l}{\rotatebox{0}{\scaleb{Datasets}}} & 
    \multicolumn{2}{c}{\rotatebox{0}{\scaleb{ECL}}} &
    \multicolumn{2}{c}{\rotatebox{0}{\scaleb{Traffic}}} &
    \multicolumn{2}{c}{\rotatebox{0}{\scaleb{Weather}}} &
    \multicolumn{1}{c}{} &
    \multicolumn{2}{c}{\rotatebox{0}{\scaleb{PEMS03}}} &
    \multicolumn{2}{c}{\rotatebox{0}{\scaleb{PEMS08}}} \\
    \cmidrule(lr){3-4} \cmidrule(lr){5-6}\cmidrule(lr){7-8} \cmidrule(lr){10-11} \cmidrule(lr){12-13}
    \multicolumn{2}{l}{\rotatebox{0}{\scaleb{Metrics}}}  & \scalea{MSE} & \scalea{MAE}  & \scalea{MSE} & \scalea{MAE}  & \scalea{MSE} & \multicolumn{1}{c}{\scalea{MAE}}  & \multicolumn{1}{c}{} & \scalea{MSE} & \scalea{MAE}  & \scalea{MSE} & \scalea{MAE} \\
    \midrule

    \multirow{5}{*}{{\rotatebox{90}{\scalebox{0.95}{MoLA-1}}}}
    & \scalea{96}
    & \scalea{0.149} & \scalea{0.239}
    & \scalea{0.401} & \scalea{0.275}
    & \scalea{0.203} & \scalea{0.245}
    & \scalea{12}
    & \scalea{0.069} & \scalea{0.175}
    & \scalea{0.081} & \scalea{0.185} \\
    
    & \scalea{192} 
    & \scalea{0.164} & \scalea{0.255} 
    & \scalea{0.425} & \scalea{0.286} 
    & \scalea{0.250} & \scalea{0.281} 
    & \scalea{24}
    & \scalea{0.099} & \scalea{0.210} 
    & \scalea{0.123} & \scalea{0.228} \\
    
    & \scalea{336} 
    & \scalea{0.182} & \scalea{0.274} 
    & \scalea{0.432} & \scalea{0.286} 
    & \scalea{0.280} & \scalea{0.298} 
    & \scalea{36}
    & \scalea{0.136} & \scalea{0.249} 
    & \scalea{0.173} & \scalea{0.271} \\
    
    & \scalea{720} 
    & \scalea{0.215} & \scalea{0.301} 
    & \scalea{0.480} & \scalea{0.320} 
    & \scalea{0.368} & \scalea{0.358} 
    & \scalea{48}
    & \scalea{0.172} & \scalea{0.284} 
    & \scalea{0.222} & \scalea{0.309} \\
    \cmidrule(lr){2-13}
    
    & \scalea{Avg} 
    & \scalea{0.177} & \scalea{0.267} 
    & \scalea{0.435} & \scalea{0.292} 
    & \scalea{0.275} & \scalea{0.296} 
    & \scalea{Avg}
    & \scalea{0.119} & \scalea{0.229} 
    & \scalea{0.150} & \scalea{0.249} \\
    \midrule

    \multirow{5}{*}{{\rotatebox{90}{\scalebox{0.95}{MoLA-2}}}}
    & \scalea{96}
    & \scalea{0.149} & \scalea{0.240}
    & \scalea{0.401} & \scalea{0.276}
    & \scalea{0.203} & \scalea{0.246}
    & \scalea{12}
    & \scalea{0.069} & \scalea{0.175}
    & \scalea{0.082} & \scalea{0.185} \\
    
    & \scalea{192} 
    & \scalea{0.163} & \scalea{0.254} 
    & \scalea{0.423} & \scalea{0.285} 
    & \scalea{0.248} & \scalea{0.280} 
    & \scalea{24}
    & \scalea{0.099} & \scalea{0.209} 
    & \scalea{0.122} & \scalea{0.228} \\
    
    & \scalea{336} 
    & \scalea{0.182} & \scalea{0.274} 
    & \scalea{0.431} & \scalea{0.286} 
    & \scalea{0.281} & \scalea{0.298} 
    & \scalea{36}
    & \scalea{0.135} & \scalea{0.248} 
    & \scalea{0.174} & \scalea{0.272} \\
    
    & \scalea{720} 
    & \scalea{0.212} & \scalea{0.299} 
    & \scalea{0.479} & \scalea{0.319} 
    & \scalea{0.369} & \scalea{0.358} 
    & \scalea{48}
    & \scalea{0.171} & \scalea{0.281} 
    & \scalea{0.218} & \scalea{0.305} \\
    \cmidrule(lr){2-13}
    
    & \scalea{Avg} 
    & \scalea{0.176} & \scalea{0.267} 
    & \scalea{0.433} & \scalea{0.292} 
    & \scalea{0.275} & \scalea{0.295} 
    & \scalea{Avg}
    & \scalea{0.118} & \scalea{0.228} 
    & \scalea{0.149} & \scalea{0.247} \\
    \midrule

    \multirow{5}{*}{{\rotatebox{90}{\scalebox{0.95}{MoLA-3}}}}
    & \scalea{96}
    & \scalea{0.150} & \scalea{0.240}
    & \scalea{0.402} & \scalea{0.276}
    & \scalea{0.201} & \scalea{0.245}
    & \scalea{12}
    & \scalea{0.069} & \scalea{0.176}
    & \scalea{0.083} & \scalea{0.186} \\
    
    & \scalea{192} 
    & \scalea{0.162} & \scalea{0.254} 
    & \scalea{0.424} & \scalea{0.287} 
    & \scalea{0.249} & \scalea{0.280} 
    & \scalea{24}
    & \scalea{0.099} & \scalea{0.210} 
    & \scalea{0.125} & \scalea{0.230} \\
    
    & \scalea{336} 
    & \scalea{0.180} & \scalea{0.272} 
    & \scalea{0.431} & \scalea{0.286} 
    & \scalea{0.280} & \scalea{0.297} 
    & \scalea{36}
    & \scalea{0.136} & \scalea{0.249} 
    & \scalea{0.175} & \scalea{0.272} \\
    
    & \scalea{720} 
    & \scalea{0.210} & \scalea{0.297} 
    & \scalea{0.481} & \scalea{0.319} 
    & \scalea{0.369} & \scalea{0.358} 
    & \scalea{48}
    & \scalea{0.171} & \scalea{0.282} 
    & \scalea{0.219} & \scalea{0.305} \\
    \cmidrule(lr){2-13}
    
    & \scalea{Avg} 
    & \scalea{0.176} & \scalea{0.266} 
    & \scalea{0.435} & \scalea{0.292} 
    & \scalea{0.275} & \scalea{0.295} 
    & \scalea{Avg}
    & \scalea{0.119} & \scalea{0.229} 
    & \scalea{0.150} & \scalea{0.248} \\
    \bottomrule
  \end{tabular}
  \end{threeparttable}
\end{table}

\end{document}